\documentclass[11pt]{article}
\usepackage{amsmath}
\usepackage{graphicx}
\usepackage{enumerate}
\usepackage{natbib}
\usepackage{url} 

\def\spacingset#1{\renewcommand{\baselinestretch}%
{#1}\small\normalsize} \spacingset{1}

\newcommand{\blind}{1}

\usepackage{float}
\usepackage{amssymb,amsthm,mathrsfs,bbm}
\usepackage{fullpage}
\usepackage{marginnote} 
\usepackage{caption}
\usepackage{subcaption}

\usepackage{enumitem}

\numberwithin{equation}{section}
\newtheorem{theorem}{Theorem}[section]
\newtheorem{lemma}{Lemma}[section]
\newtheorem{corollary}{Corollary}[section]

\newtheorem{proposition}{Proposition}[section]

\newtheorem{remark}{Remark}[section]

\usepackage[usenames]{color}
\definecolor{plum}{rgb}{.4,0,.4}
\definecolor{BrickRed}{rgb}{0.6,0,0}
\usepackage{commath}
\usepackage[normalem]{ulem}
\usepackage{mathabx}
\usepackage[plainpages=false,pdfpagelabels,colorlinks=true,linkcolor=BrickRed,citecolor=plum]{hyperref}

    \def\ddefloop#1{\ifx\ddefloop#1\else\ddef{#1}\expandafter\ddefloop\fi}

    \def\ddef#1{\expandafter\def\csname c#1\endcsname{\ensuremath{\mathcal{#1}}}}
    \ddefloop ABCDEFGHIJKLMNOPQRSTUVWXYZ\ddefloop

    \def\ddef#1{\expandafter\def\csname s#1\endcsname{\ensuremath{\mathsf{#1}}}}
    \ddefloop ABCDEFGHIJKLMNOPQRSTUVWXYZ\ddefloop

    \def\E{\mathop{\mathbf{E}}}

    \def\argmin{\operatornamewithlimits{arg\,min}}

    \def\deq{:=}

    \def\bd#1{\mathbf{#1}}
    \def\bx{\bd{x}}
	\def\by{\bd{y}}

\usepackage{tabularx,multirow,booktabs}
\usepackage[symbol]{footmisc}
\usepackage{xr}

\begin{document}


\if1\blind
{
  \title{\bf Training Neural Networks as Learning Data-adaptive Kernels: Provable Representation and Approximation Benefits}
  \author{Xialiang Dou \\
   Department of Statistics, University of Chicago
    \and 
    Tengyuan Liang\thanks{
    Liang gratefully acknowledges support from the George C. Tiao Fellowship.}\\
    Booth School of Business, University of Chicago}
	\date{}
  \maketitle
} \fi

\if0\blind
{
  \bigskip
  \bigskip
  \bigskip
  \begin{center}
    {\LARGE\bf Training Neural Networks as Learning Data-adaptive Kernels: Provable Representation and Approximation Benefit}
\end{center}
  \medskip
} \fi

\bigskip
\begin{abstract}
Consider the problem: given the data pair $(\mathbf{x}, \mathbf{y})$ drawn from a population with $f_*(x) = \mathbf{E}[\mathbf{y} | \mathbf{x} = x]$, specify a neural network model and run gradient flow on the weights over time until reaching any stationarity. How does $f_t$, the function computed by the neural network at time $t$, relate to $f_*$, in terms of approximation and representation? What are the provable benefits of the adaptive representation by neural networks compared to the pre-specified fixed basis representation in the classical nonparametric literature? We answer the above questions via a dynamic reproducing kernel Hilbert space (RKHS) approach indexed by the training process of neural networks. Firstly, we show that when reaching any local stationarity, gradient flow learns an adaptive RKHS representation and performs the global least-squares projection onto the adaptive RKHS, simultaneously. Secondly, we prove that as the RKHS is data-adaptive and task-specific, the residual for $f_*$ lies in a subspace that is potentially much smaller than the orthogonal complement of the RKHS. The result formalizes the representation and approximation benefits of neural networks. Lastly, we show that the neural network function computed by gradient flow converges to the kernel ridgeless regression with an adaptive kernel, in the limit of vanishing regularization. The adaptive kernel viewpoint provides new angles of studying the approximation, representation, generalization, and optimization advantages of neural networks.
\end{abstract}

\noindent%
\vfill

\newpage
\spacingset{1.5} 

\section{Introduction}
\label{sec:intro}

Consider i.i.d. data pairs drawn from a joint distribution $(\bx, \by) \sim P = P_{x} \times P_{y|x}$ on the space $\mathcal{X} \times \mathcal{Y}$. 
At the intersection of statistical learning theory \citep{vapnik1998statistical} and approximation theory \citep{cybenko1989approximation}, the following \textit{approximation} problem requires to be first understood, before any further statistical results to be established. 
For a model class $\cF$, one is interested in whether there exists $f \in \cF: \mathcal{X} \rightarrow \mathcal{Y}$ such that the population squared loss is small,
\begin{align}\label{appr_prob}
     L(f) = \,\E_{(\bx,\bd{y})\sim P}\frac{1}{2}\left( \bd{y}-f(\bx) \right)^2 
     = \E_{\bx \sim P_x} \frac{1}{2} \left( f_*(\bx) - f(\bx) \right)^2 + \E_{(\bx,\bd{y})\sim P} \frac{1}{2} \left( \by - f_*(\bx) \right)^2,
\end{align}
with the conditional expectation (or Bayes estimator) defined as $f_*(x) \deq \E[\bd{y}|\bd{x} = x]$. Eqn.~\eqref{appr_prob} generally reads as approximating $f_*$ in the mean squared error sense.

Statistically, researchers approach the above question mainly in two ways. The first is by assuming that the conditional expectation $f_*$ lies in the correct model class $\cF$. For example, say $\cF$ consists of linear models or splines with a particular order of smoothness, or more broadly functions lying in a reproducing kernel Hilbert space (RKHS). Conceptually, this ``well-specification'' assumption requires substantial knowledge about what model class $\cF$ might be suitable for the regression task at hand, which is often unavailable in practice. Within each framework, minimax optimal rates and extensive study have been established in \citep{stone1980optimal,wahba1990spline}.
The second way, which extends the first approach further, considers all $f_*$ under some mild conditions. Building upon certain \textit{universal approximation theorem}, one studies a sequence of model classes $\cF_{\epsilon}$ called sieves with $\epsilon$ changing \citep{geman1982nonparametric}, such that the class $\cF_{\epsilon}$ contains an $\epsilon$-approximation to any $f_*$ under some metric. A final result usually requires a careful balancing of the approximation and stochastic error by tuning $\epsilon$. Particular cases for the latter approach include polynomials (Stone-Weierstrass, Bernstein), radial-basis \citep{park1991universal, niyogi1996relationship}, and two-layer and multi-layer neural networks \citep{cybenko1989approximation,hornik1989multilayer, anthony2009neural, rahimi2008random, daniely2016toward, bach2017breaking, farrell2018deep, koehler2018representational, poggio2017and}.

However, the following significant drawbacks of the above current theory make it inadequate to present an \textit{adaptive} and realistic explanation of the practical success of \textit{neural networks}. 
Firstly, the function computed in practice could be very different from that claimed in the approximation theory, either by the existence or by constructions. To see this, consider the multi-layer neural networks. It is hard to conceive that the function, computed in practice via now-standard stochastic gradient descent (SGD) training procedure, is close to the one asserted by the universal approximation results. Secondly, in practice, researchers usually explore different model classes $\cF$ to learn which representation best suits the data. For example, using different kernels machines, random forests, or specify certain architectures then run SGD on neural networks. In this case, strictly speaking, the choice of the model class depends on the data in an \textit{adaptive} way, without prior knowledge about the basis. There have been substantial advances made to address the above two concerns --- for instance, \cite{jones1992simple} on the first and \cite{huang2008risk, barron2008approximation} on the second --- for $\cF$ being a linear span of a library of candidate functions (union of various set of basis that can be correlated), with greedy selection rules. Nevertheless, the current theory still falls short of describing the approximation and adaptivity for the non-convex and possibly non-smooth gradient descent training on all-layer weights of the neural networks, as done in practice.

We take a step to bridge the above mismatch in the current theory and practice for neural networks and to establish a theoretical framework where the model classes adapt to the data. In particular, we answer the following \textit{algorithmic approximation} question:

\begin{quote}
    Given data pair $(\bx, \by) \sim P$, denote $f_*(x) = \E[\by | \bx = x]$. Specify a neural networks model, and run gradient flow until any stationarity ($t \rightarrow \infty$). Denote the computed function to be $f_t(x)$. How does $f_t(x)$ relate to $f_*(x)$, in terms of approximation and representation?
\end{quote}

Also, we aim to formalize and shed light on the \textit{representation benefits} of neural networks:

\begin{quote}
    What are the provable benefits of the adaptive representation learned by training neural networks compared to the classical nonparametric pre-specified fixed basis representation?
\end{quote}

The intimate connection between two-layer neural networks and reproducing kernel Hilbert spaces (RKHS) has been studied in the literature, see \cite{rahimi2008random, cho2009kernel, daniely2016toward, bach2017breaking, jacot2018neural}. However, to the best of our knowledge, known results are mostly based on a \textit{fixed} RKHS (in our notation $K_0$ in Section~\ref{sec:K0}). In that sense, random features for kernel learning \citep{rahimi2008random, rahimi2009weighted, rudi2017generalization} can be viewed as neural networks with fixed random sampled first layer weights, and tunable second layer weights. From the neural networks side, \cite{rotskoff2018neural, mei2018mean, sirignano2019mean} study the mean-field theory for two-layer neural networks, and \cite{jacot2018neural, du2018gradient,chizat2018note, ghorbani2019linearized} study the linearization of neural networks around the initialization and draw connections to RKHS $\cK_0$ in various over-parametrized settings. In contrast, we will establish a general theory with the \textit{dynamic} and \textit{data-adaptive} RKHS $\cK_t$ obtained via training neural networks, with standard gradient flow on weights of \textit{both} layers. Connections and distinctions to the literature that motivates our study are further discussed with details in Section~\ref{sec:kernel}. As a distinctive feature of the adaptive theory, we emphasize that all $f_* \in L^2(P_x)$ is considered, without pre-specified structural assumptions.

\subsection{Problem Formulation}
\label{sec:formulation}

In this paper, we consider the time-varying function $f_t$ to approximate $f_*$, parametrized by a two-layer rectified linear unit (ReLU) neural network (NN). 
\begin{align}\label{nn}
    f_t(x) := \sum_{j=1}^{m} w_j(t)\sigma(x^T u_j(t)). 
\end{align}
The time index $t$ corresponds to the evolution of parameters driven by the gradient flow/descent (GD) training dynamics. Here each individual pair $(w_j \in \mathbb{R},u_j \in \mathbb{R}^d)$ in the summation is associated with a \textit{neuron}. Consider the gradient flow as the training dynamics for the weights of the neurons: for the loss function $\ell(y, f)= (y -f)^2/2$ and the random variable $\bd{z} := (\bx,\bd{y})$, the parameters $(w_j,u_j)$ evolve with time as follows
\begin{align}
    \frac{dw_j(t)}{dt} = -\mathbf{E}_{\bd{z}}\left[\frac{\partial \ell(\by,f_t)}{\partial f}\sigma(\bx^Tu_j(t))\right],\quad
    \frac{du_j(t)}{dt} = -\mathbf{E}_{\bd{z}}\left[\frac{\partial \ell(\by,f_t)}{\partial f}w_j(t)\mathbbm{1}_{\bx^Tu_j(t)\geq0}\bx\right]. \label{eq:training_2}
\end{align}

Equivalently, we can rewrite the function computed by NN at time $t$ as 
\begin{align}\label{def:convex-nn}
    f_t(x) := \int \sigma(x^Tu) \tau_{t}(du),
\end{align}
where $\tau_{t} = \sum_{j=1}^m w_j(t) \delta_{u_j(t)}$ is a signed combination of delta measures. We will define a careful rescaling of $\tau_{t}$ denoted as $\rho_t$ (Eqn.~\eqref{eq:sign-measure}), then derive the corresponding distribution dynamic for $\rho_t$ driven by the gradient flow later in Section~\ref{sec:pde-char}. The rescaled formulation naturally extends to the infinite neurons case with $m \rightarrow \infty$.

In this paper, by considering various distributions of $\bd{z}$, we study two following problems: approximation and empirical risk minimization (ERM). 

\noindent {\bf Function Approximation:} The data pair $\bd{z} \sim P$ is sampled from the population joint distribution. We are going to answer
how $f_t$ \textit{approximates} $f_*(x) = \E[\by | \bx = x]$ in function spaces, induced by the gradient flow on neuron weights
\begin{align}
    \label{eq:approx}
    \E_{\bd{z} \sim P} (\by - f_t(\bx))^2 = \| f_t - f_* \|_{L^2_\mu}^2 + \E_{\bd{z} \sim P} (\by - f_*(\bx))^2\enspace.~
\end{align}
Here we denote $\mu:= P_x$, and remark that all $f_* \in L^2_\mu$ are considered without additional assumptions. 

\noindent {\bf ERM and Interpolation:} The data pair $\bd{z} \sim \frac{1}{n} \sum_{i=1}^n \delta_{\bx = x_i, \by = y_i}$ follows the empirical distribution. We will study gradient flow for the ERM
\begin{align}
    \label{eq:erm}
    \frac{1}{2n} \sum_{i=1}^n (y_i - f_t(x_i))^2\enspace.
\end{align}
In this case, the target reduces to $\widehat{\E}[\by|\bx = x_i] = y_i$ with $\widehat{\E}$ as the empirical expectation. When the minimizer of Eqn.~\eqref{eq:erm} achieves the zero loss, we call it the \textit{interpolation} problem \citep{zhang2016understanding, belkin2018understand, ma2017power, liang2018just, rakhlin2018consistency, belkin2018reconciling}.
Here we are interested in when and how $f_t(x_i)$ \textit{interpolates} $y_i$, for $1\leq i \leq n$. 

Finally, we remark that in practice, extending the gradient flow results to the (1) positive step size GD, and (2) mini-batch stochastic GD, are standalone interesting research topics. The reasons are that the optimization is non-smooth for the ReLU activation and that the interplay between the batch size and step size is less transparent in non-convex problems.

\section{Preliminaries and Summary}
\label{sec:pre}

\subsection{Notations}

We use the boldface lower case $\bd{x}$ to denote a random variable or vector. The normal letter $x$ can either be a scaler or a vector when there is no confusion. The transpose of a matrix $\bd{A}$, resp. vector $u$ is denoted by $\bd{A}^T$, resp. $u^T$. $\bd{A}^+$ denotes the Moore–Penrose inverse. For $n\in\mathbb{N}$, let $[n]:=\{1,\dots,n\}$. We use $\mathbf{A}[i,j]$ to denote the $i,j$-th entry of a matrix. We denote $\mathbbm{1}_{\mathcal{D}}$ as the indicator function of set $\mathcal{D}$. We call symmetric positive semidefinite functions $K(\cdot,\cdot), H(\cdot, \cdot):\, \mathcal{X}\times\mathcal{X} \rightarrow \mathbb{R}$ kernels, and use calligraphy letter $\mathcal{K},\mathcal{H}$ to denote Hilbert spaces. We use $\langle f,g\rangle_{\mu} = \int f(x)g(x)\mu(dx)$ to denote the inner product in $L^2_\mu$ (or $L^2(P_x)$). $\hat{\mu}$ denotes the empirical distribution for $\mu$. Notation $\mathbf{E}_{\bx}$ is the expectation w.r.t random variable $\bx$, and $\E_{\bx, \bd{\tilde{x}}} h(\bx, \bd{\tilde{x}}) = \int \int h(x, \tilde{x}) \mu(dx) \mu(d \tilde{x})$. For a signed measure $\rho = \rho_+ - \rho_-$ with the positive and negative parts, define $|\rho| = \rho_+ + \rho_-$.

\subsection{Preliminaries}
\label{sec:preliminary}

We use the signed measure $\rho_t$, defined by the neuron weights at training time $t$ collectively, to construct a \textit{dynamic RKHS}. The mathematical definition of $\rho_t$ is deferred to Section~\ref{sec:K0} and \ref{sec:pde-char} (specifically, Eqn.~\eqref{eq:sign-measure}). The stationary signed measure at $t\rightarrow \infty$ is denoted as $\rho_\infty$. 
For completeness we walk through the construction of the dynamic kernel and RKHS with $\rho_{t}$. Define the linear operator $\mathcal{T}: L^2_\mu(x) \rightarrow L^2_{|\rho_t|}(\Theta)$, such that for any $f(x) \in L^2_\mu(x)$
\begin{align*}
    (\mathcal{T} f)(\Theta) := \int f(x) \| \Theta \| \sigma(x^T\Theta) \mu(dx), ~\forall \Theta\in {\rm supp}(\rho_t).
\end{align*}
One can define the adjoint operator $\mathcal{T}^\star: L^2_{|\rho_t|}(\Theta) \rightarrow L^2_\mu(x)$, such that for $p(\Theta) \in L^2_{|\rho_t|}(\Theta)$, 
\begin{align*}
    (\mathcal{T}^\star p)(x) := \int p(\Theta) \| \Theta \| \sigma(x^T\Theta)  |\rho_t|(d \Theta).
\end{align*}
    Note that both $\mathcal{T}$ and $\mathcal{T}^\star$ are compact operators under the finite total variation and compact support assumptions. For the finite neurons case \eqref{nn}, the operator is of finite rank. We define the compact integral operator $\mathcal{T}^\star\mathcal{T}$ with the corresponding kernel
    \begin{align}\label{def:rkhs-kernel}
        H_t(x, \tilde{x}) = \int \| \Theta \|^2 \sigma(x^T\Theta) \sigma(\tilde x^T\Theta) |\rho_t|(d \Theta), \text{ and  }\,
        (\mathcal{T}^\star\mathcal{T} f) (x) := \int H_t(x,\tilde x)f(\tilde x)\mu(d \tilde x).
    \end{align}
    The dynamic RKHS $\mathcal{H}_t$ can be readily constructed via $H_t$. Let the eigen decomposition of $\mathcal{T}^\star\mathcal{T}$ be the countable sum
    $
        \mathcal{T}^\star\mathcal{T} = \sum_{i=1}^E \lambda_ie_ie_i^*.
    $
    Here $E$ can be a nonnegative integer or $\infty$, and $\lambda_i>0$. $e_i$ without confusion can represent either an eigen function or a linear functional. Similarly, we have the singular value decomposition for $\mathcal{T} = \sum_{i=1}^E \sqrt{\lambda_i}t_ie_i^*.$ and $\mathcal{T}^\star$ as well. For a detailed discussion, see e.g. \cite{casselman2014essays}. Again, $t_i$ is a function in $L^2_{|\rho_{t}|}(\Theta)$ or a linear functional. The RKHS can be specified as follows.
    \begin{align*}
        \mathcal{H}_t = \left\{ h \mid h(x) = \sum_i h_ie_i(x) \text{, } \sum_i \frac{h^2_i}{\lambda_i} < \infty \right\}.
    \end{align*}
    We refer to $H_\infty$ as the stationary RKHS kernel, and $\mathcal{H}_\infty$ as the stationary RKHS.
 One can view that the gradient flow training dynamics --- on the parameters of NN --- induces a sequence of functions $\{ f_t: t \geq 0\}$ and dynamic RKHS $\{ \mathcal{H}_t: t \geq 0\}$, indexed by the time $t$.

\subsection{Organization and Summary}
\label{sec:summary-all}

\begin{table}[ht!]
    \caption{Nature of the results studied in this paper.}
    \newcolumntype{Y}{>{\centering\arraybackslash}X}
    \small
    \begin{tabularx}{\columnwidth}{@{}  m{2cm} | Y | Y @{}}
        \toprule
         &  finite neurons $m$ & infinite neurons $m\rightarrow \infty$ \\
        \midrule
        finite samples $n$ & Interpolation (finite rank kernel, Thms.~\ref{thm:proj-solution}, \ref{thm:gap-decomposition} \& Prop.~\ref{prop:interpolation}) & Interpolation (finite rank kernel, Thms.~\ref{thm:proj-solution}, \ref{thm:gap-decomposition} \& Prop.~\ref{prop:interpolation}) \\
        \midrule
        infinite samples $n\rightarrow \infty$ & Approximation (finite rank kernel, Thms.~\ref{thm:proj-solution} \& \ref{thm:gap-decomposition}) &  Approximation (possibly universal kernel\footnotemark[2], Thms.~\ref{thm:proj-solution} \& \ref{thm:gap-decomposition})\\
        \bottomrule
    \end{tabularx}
    \label{table:summary}
\end{table}
\footnotetext[2]{Whether the kernel is universal in the $m, n\rightarrow \infty$ case still depends on $f_*$ and the data distribution $P$. See the simulations of \cite{maennel2018gradient}.}

We will prove three results, which are summarized informally in this section (see also Table~\ref{table:summary}). We remark that Theorems~\ref{thm:proj-solution} and \ref{thm:gap-decomposition} are stated for the approximation problem. However, as done in Corollary~\ref{cor:proj-solution}, by substituting $\mathcal{P}, \mu$ by the empirical counterparts, one can easily state the analog for the ERM problem. Recall $f_*(x) = \E[\bd{y}|\bd{x} = x]$.

\noindent \textbf{Gradient flow on NN converges to projection onto data-adaptive RKHS.}
Theorem~\ref{thm:proj-solution} shows that as done in practice training NN with simple gradient flow, in the limit of any \textit{local} stationarity, learns the adaptive representation, and performs the \textit{global} least squares projection simultaneously. Define $f_\infty = \lim_{t \rightarrow \infty} f_t$ as the function computed by ReLU networks (defined in \eqref{nn}, or more generally in \eqref{def:para-nn}) until any stationarity of the gradient flow dynamics (defined in \eqref{eq:training_2}, with the squared loss) for the population distribution $(\bx, \by)\sim P$ . Define the corresponding stationary RKHS $\mathcal{H}_\infty = \lim_{t \rightarrow \infty} \mathcal{H}_t$ (defined in \eqref{def:rkhs-kernel}). 
\begin{quote}[Informal version of Thm.~\ref{thm:proj-solution}]~
    Consider $f_* \in L^2_\mu$, for any local stationarity of the gradient flow dynamics \eqref{eq:training_2} on the weights of neural networks \eqref{nn}, the function computed by NN at stationarity $f_\infty$ satisfies
    \begin{align*}
        f_\infty \in \argmin_{g \in \mathcal{H}_\infty} \| f_* - g \|_{L^2_\mu}^2.
    \end{align*}
\end{quote}

\noindent \textbf{Representation benefits of data-adaptive RKHS.} 
Theorem~\ref{thm:gap-decomposition} illustrates the provable benefits of the learned data-adaptive representation/basis $\mathcal{H}_\infty$. We emphasize that $\mathcal{H}_\infty$, as obtained by training neural networks on the data $(\bx, \by)\sim P$, depends on the data in an implicit way such that there are advantages of representing and approximating $f_*$.

\begin{quote}[Informal version of Thm.~\ref{thm:gap-decomposition}]~
    Consider $f_* \in L^2_\mu$ and the same setup as Theorem~\ref{thm:proj-solution}. Decompose $f_*$ into the function $f_\infty$ computed by the neural network and the residual $\Delta_\infty$
    \begin{align*}
        f_* = f_\infty + \Delta_\infty.
    \end{align*}
    Then there is another RKHS (defined in \eqref{eq:rkhs-gd}) $\mathcal{K}_\infty \supset \mathcal{H}_\infty$, such that
    \begin{align*}
        f_\infty \in \mathcal{H}_\infty ,\quad
        \Delta_\infty \in \text{Ker}(\mathcal{K}_\infty) \subset \text{Ker}(\mathcal{H}_\infty),
    \end{align*} with a gap in the spaces $\mathcal{H}_\infty \oplus \text{Ker}(\mathcal{K}_\infty) \neq L^2_\mu$.
\end{quote}

\noindent \textbf{Convergence to Ridgeless regression with adaptive kernels.} Proposition~\ref{prop:interpolation} establishes that in the vanishing regularization $\lambda \rightarrow 0$ limit, the neural network function computed by gradient flow converges to the kernel ridgeless regression with an adaptive kernel (denoted as $\widehat{f}^{\rm rkhs}_{\infty}(x)$). Consider using the gradient flow on the weights of the neural network function $f_t(x) = \sum_{j=1}^m w_j(t) \sigma(x^T u_j(t))$, to solve the $\ell_2$-regularized ERM
    \begin{align*}
        \frac{1}{2n} \sum_{i=1}^n (y_i - f_t(x_i))^2 + \frac{\lambda}{2m} \sum_{j=1}^m \left[ w_j(t)^2 + \| u_j(t) \|^2 \right] \enspace.
    \end{align*}
Denote the function computed by NN at any local stationarity of ERM as $\widehat{f}^{{\rm nn}, \lambda}(x)$, we answer the extrapolation question at a new point $x$, with the generalization error discussed in Prop.~\ref{prop:adaptive-generalization}. The result is extendable to the infinite neurons case.
\begin{quote}[Informal version of Prop.~\ref{prop:interpolation}]~
    Consider only the bounded assumption on initialization that $|w_j^2(0) - \|u_j\|^2(0)| <\infty$ for all $1\leq j\leq m$.
    At stationarity, denote the corresponding adaptive kernel as $\widehat{H}_{\infty}^{\lambda}$.
    The neural network function $\widehat{f}^{{\rm nn}, \lambda}_{\infty} (x)$ has the following expression, 
    \begin{align*}
         \lim_{\lambda \rightarrow 0} \widehat{f}^{{\rm nn}, \lambda}_{\infty} (x) = \widehat{H}_{\infty}(x, X) \widehat{H}_{\infty}(X, X)^{+} Y =: \widehat{f}^{\rm rkhs}_{\infty}(x) ~\text{(ridgeless regression with kernel $\widehat{H}_\infty$)}.
    \end{align*}
\end{quote}


\section{Main Results: Benefits of Adaptive Representation}
\label{sec:adaptive}

We formally state two main results of the paper, Theorem~\ref{thm:proj-solution} and Theorem~\ref{thm:gap-decomposition} below. 

\subsection{Gradient Flow, Projection and Adaptive RKHS}

We study how the function $f_t$ computed from gradient flow on NN represents $f_*$ when reaching any stationarity, under the squared loss. 
Consider the gradient flow dynamics \eqref{def:training-infinite} reaching \textit{any stationarity}. Assume that the corresponding signed measure in \eqref{eq:sign-measure} satisfies $\text{TV}(\rho_{\infty}) < \infty$ with a compact support. The mathematical details about $\rho_\infty$ are postponed to Section~\ref{sec:pde-char}. We employ the notation $\rho_{\infty}$ since reaching stationarity can be viewed as $t\rightarrow \infty$.

We would like to emphasize that this stationary signed measure $\rho_{\infty}$ is \textit{task adaptive}: it implicitly depends on the regression task $f_*$ and the data distribution $P$, rather than being pre-specified by the researcher as in \cite{bach2017breaking, daniely2016toward, cho2009kernel}. With the RKHS established in Section~\ref{sec:preliminary}, we are ready to state the following theorem.

    \begin{theorem}[Approximation]
        \label{thm:proj-solution}
        For any conditional mean $f_*(x) = \E[\by|\bx =x] \in L^2_\mu$,
        consider solving the approximation problem \eqref{eq:approx}, with the ReLU NN function $f_t$ defined in \eqref{nn} where $w_j(t)$ and $\theta_j(t)$ are the weights for $t\geq 0, 1\leq j \leq m$. For any signed measure $\rho_0$ with ${\rm TV}(\rho_0) < \infty$, consider the infinitesimal initialization weights $\ u_j(0) = \Theta_j/\sqrt{m}$, and $w_j(0) = {\rm sgn}(\rho_0(\Theta_j)) \| \Theta_j \|/\sqrt{m}$, with $\Theta_j \sim \rho_0$ sampled independently. When the training dynamics \eqref{eq:training_2} reaches any stationarity, it defines a stationary signed measure $\rho^{(m)}_\infty$ (on the collective weights) with ${\rm TV}(\rho^{(m)}_\infty) < \infty$, and a corresponding stationary RKHS $\mathcal{H}_\infty$ with the kernel defined in Eqn.~\eqref{def:rkhs-kernel}, such that:
        \begin{enumerate}
            \item the function computed by neural networks at stationarity has the form
            \begin{align}
                \label{eq:infinite-sign-measure}
                f_{\infty}(x)  = \int \|\Theta\|\sigma(x^T\Theta)\rho^{(m)}_{\infty}(d\Theta) \enspace;
            \end{align}
            \item $f_{\infty}$ is a global minimizer of approximating $f_*$ within the RKHS $\cH_\infty$
            \begin{align}
                \label{eq:rkhs-proj}
                f_{\infty} \in \argmin_{g\in\mathcal{H}_\infty} \|f_* - g \|_{L^2_\mu}^2 \enspace.
            \end{align}
        \end{enumerate}
        In addition, the same results extend to the infinite neurons case with $m\rightarrow \infty$ where the limit for $\rho^{(m)}_\infty$ can be defined in the weak sense.
    \end{theorem}

    \begin{remark}
    \rm 
    The above theorem shows that $\lim_{t\rightarrow \infty} f_t$ obtained by training on two-layer weights over time until any stationarity, 
    is the same as projecting $f_*$ onto the stationary RKHS $\mathcal{H}_\infty$. The projection is the solution to the classic nonparametric least squares, had one known the adaptive representation $\mathcal{H}_\infty$ beforehand. Conceptually, this is \textit{distinct} from the theoretical framework in the current statistics and learning theory literature: we do not require the structural knowledge about $f_*$ (say, smoothness, sparsity, reflected in $\cF$). Instead, we run gradient descent on neural networks to learn an adaptive representation for $f_*$, and show how the computed function represents $f_*$ in this adaptive RKHS $\mathcal{H}_\infty$.

    In other words, as done in practice training NN with simple gradient flow, in the limit of any \textit{local} stationarity, learns the adaptive representation, and performs the \textit{global} least-squares projection simultaneously.
    Training NN is learning a dynamic representation (quantified by $\mathcal{H}_t$), at the same time updating the predicted function $f_t$, as shown in Fig.~\ref{fig:projection-rkhs}.
        
    A final note on the infinite neuron case: for any fixed time $t$, with the proper random initialization, setting $m\rightarrow \infty$ defines a proper distribution dynamics on the weak limit $\rho_t$ shown in Lemma~\ref{lem:pdf-sign-measure}. Then set $t \rightarrow \infty$ to obtain the stationarity RKHS $\mathcal{H}_\infty$.  
    \end{remark}
    \begin{figure}[H]
        \centering
        \includegraphics[width = 0.4\textwidth]{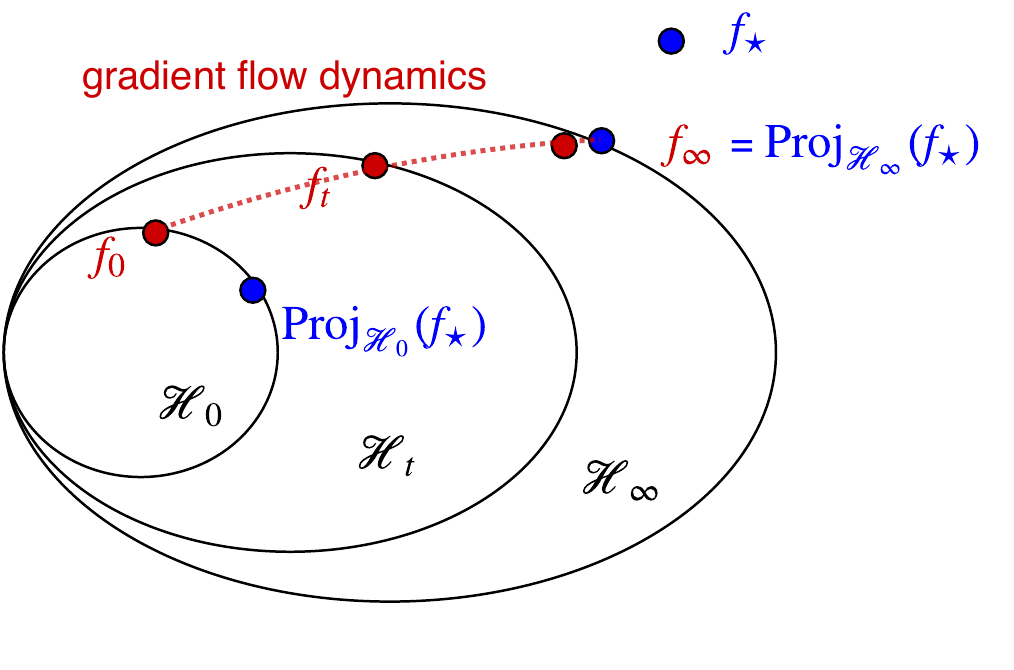}
        \caption{\small Illustration of Theorem~\ref{thm:proj-solution}. Red dotted line denotes the function $f_t$ computed along the gradient flow dynamics on the weights of NN. Along training, one learns a sequence of dynamic RKHS representation $\mathcal{H}_t$'s. Over time, $f_t$ converges to the projection of $f_*$ onto $\mathcal{H}_\infty$. We emphasize that the initial function $f_0$ computed by NN is very different from the projection of $f_*$ onto the initial RKHS $\mathcal{H}_0$.}
        \label{fig:projection-rkhs}
    \end{figure}

    From the above, we have the following natural decomposition, 
    \begin{align}\label{eq:null-rkhs}
        \Delta_\infty(x) = f_*(x) - f_{\infty}(x) \in \text{Ker}(\mathcal{H}_\infty).
    \end{align}
    Surprisingly, as we show in the next section, $\Delta_\infty$ actually lies in a smaller subspace of $\text{Ker}(\mathcal{H}_\infty)$, characterized by $\text{Ker}(\mathcal{K}_\infty)$. We call this the \textit{representation and approximation benefits} of the data-adaptive RKHS learned by training neural networks. 

    Before moving next, we briefly discuss the above theorem when applied to the empirical measure, to solve the ERM problem. First, as a direct corollary, the following holds.
    \begin{corollary}[ERM]
        \label{cor:proj-solution}
        Consider the ERM problem \eqref{eq:erm}, with the other settings the same as in Theorem~\ref{thm:proj-solution}. One can define the finite dimensional RKHS $\widehat{\cH}_\infty$ (at most rank $n$) as in \eqref{def:rkhs-kernel} with $\widehat{\mu}=\frac{1}{n}\sum_{i=1}^n \delta_{x_i}$ substituting $\mu$. When reaches any stationarity, the solution satisfies
        \begin{align*}
            \widehat{f}_{\infty} \in \argmin_{g\in \widehat{\cH}_\infty} \frac{1}{n}\sum_{i=1}^n (y_i - g(x_i))^2 \enspace.
        \end{align*}
    \end{corollary}
    More importantly, we will show in Proposition~\ref{prop:interpolation} that the function computed by training neural networks with gradient descent on the empirical risk objective $\widehat{f}_{\infty}(x)$ until any stationarity (with vanishing $\ell_2$ regularization), can be shown to be the \textbf{kernel ridgeless regression} with the \textbf{data-adaptive RKHS} $\widehat{\cH}_{\infty}$. Hence, studying the out of sample performance for GD on NN reduces to the generalization of kernel ridgeless regression with adaptive kernels.


\subsection{Representation Benefits of Adaptive RKHS}

    We now define another adaptive RKHS $\mathcal{K}_\infty$ named as the GD kernel, which turns out to be different from $\mathcal{H}_\infty$ in \eqref{def:rkhs-kernel}. Interestingly, the difference in these two kernels sheds light on the representation benefits of the adaptive RKHS. The new RKHS $\mathcal{K}_\infty$ is motivated by the gradient training dynamics. Recall the associated signed measure $\rho_\infty$ at the stationarity, The GD kernel is defined as
        \begin{align}
            \label{eq:rkhs-gd}
            K_\infty(x, \tilde{x}) = \int \left(\|\Theta\|^2\mathbbm{1}_{x^T\Theta\geq0}\mathbbm{1}_{\tilde x^T\Theta\geq0}x^T \tilde x + \sigma(x^T\Theta)\sigma(\tilde x^T\Theta) \right) |\rho_\infty|(d\Theta) \neq H_\infty(x, \tilde{x})
        \end{align}
        which is different than the stationary RKHS kernel $H_\infty$ in \eqref{def:rkhs-kernel}.
        We use $\mathcal{K}_t: L^2_\mu(x) \rightarrow L^2_\mu(x)$ to denote the integral operator associated with $K_t$,
        \begin{align*}
            (\mathcal{K}_t f)(x) := \int  K_t( x, \tilde{x} ) f(\tilde{x})\mu(d\tilde{x}).
        \end{align*}
        With a slight abuse of notation, we denote the corresponding RKHS to be $\mathcal{K}_t$ as well. Now we are ready to state the main theorem on the representation benefits.

    \begin{theorem}[Representation Benefits]
        \label{thm:gap-decomposition}
        Consider $f_* \in L^2_\mu$ and the same setting as in Theorem~\ref{thm:proj-solution}. 
    Consider the approximation problem \eqref{eq:approx} with either finite or infinite neurons, and the gradient flow dynamics \eqref{def:training-infinite} (equivalently \eqref{eq:training_2}) with data pair $(\bx, \by)\sim P$ drawn from the population distribution.
    When reaching any stationary signed measure $\rho_{\infty}$, $f_*$ is decomposed into the function $f_\infty$ computed by the neural network and the residual $\Delta_\infty$
    \begin{align*}
        f_* = f_{\infty} + \Delta_\infty.
    \end{align*}
    Recall the RKHS $\mathcal{H}_\infty$ in \eqref{def:rkhs-kernel} and the GD RKHS $\mathcal{K}_\infty$ in \eqref{eq:rkhs-gd}, all learned from the data $(\bx, \by)\sim P$ and $f_*$ adaptively. The following holds,
    \begin{align*}
        f_\infty \in \mathcal{H}_\infty, \quad
        \Delta_\infty \in \text{Ker}(\mathcal{K}_{\infty}) \subset \text{Ker}(\mathcal{H}_\infty),
    \end{align*}
    with $\mathcal{H}_\infty \oplus \text{Ker}(\mathcal{K}_\infty) \neq L^2_\mu$.
    In other words, GD on NN decomposes $f_*$ into two parts, and each lies in a space that is NOT the orthogonal complement to the other.
    \end{theorem}

        \begin{remark}
            \rm 
            As we can see $\text{Ker}(K_\infty)$ and $\text{Ker}(H_\infty)$ are not the same. Therefore, the decomposition $f_{\infty} + \Delta_{\infty}$ is not a trivial orthogonal decomposition to the RKHS $\mathcal{H}_\infty$ and its complement.
        
            Recall Theorem~\ref{thm:proj-solution}, projecting $f_*$ to the RKHS $\mathcal{H}_\infty$ with the data-adaptive kernel
            \begin{align*}
                H_\infty(x, \tilde x) = \int  \sigma(x^T\Theta) \sigma(\tilde x^T\Theta) |\rho_\infty|(d \Theta)
            \end{align*}
            associated with $|\rho_\infty|$ is the same as the function constructed by neural networks (GD limit as $t\rightarrow \infty$).
            However, the residual lies in a possibly much smaller space due to Theorem~\ref{thm:gap-decomposition}, which is the null space of the RKHS $\mathcal{K}_\infty$ 
            \begin{align*}
                K_\infty(x, \tilde{x}) = \int \left(\|\Theta\|^2\mathbbm{1}_{x^T\Theta\geq0}\mathbbm{1}_{\tilde x^T\Theta\geq0}x^T \tilde x + \sigma(x^T\Theta)\sigma(\tilde x^T\Theta) \right) |\rho_\infty|(d\Theta).
            \end{align*} 
        
        In other words, as the learned adaptive basis $\mathcal{H}_\infty$ (from GD) depends on the data distribution and the task $f_*$ implicitly, it has the advantage of representing $f_*$ by squeezing the residual into a smaller subspace in the null space of $\mathcal{H}_\infty$. A pictural illustration can be found in Fig.~\ref{fig:adaptive-basis}. This representation and approximation benefit helps with explaining the better interpolation results obtained by neural networks \citep{zhang2016understanding,belkin2018understand,liang2018just,belkin2018reconciling}: (1) the adaptive basis is tailored for the task $f_*$, thus the residual/interpolation error lies in a smaller space; (2) in view of the ODE in Corollary~\ref{cor:signed-kernel}, the second layer of NN adds \textit{implicit regularization} to the smallest eigenvalues of $K_t$, thus improving the converging speed of $\Delta_t$ to zero.
         \end{remark}
    
        \begin{figure}[H]
            \centering
            \includegraphics[width = 0.75\textwidth]{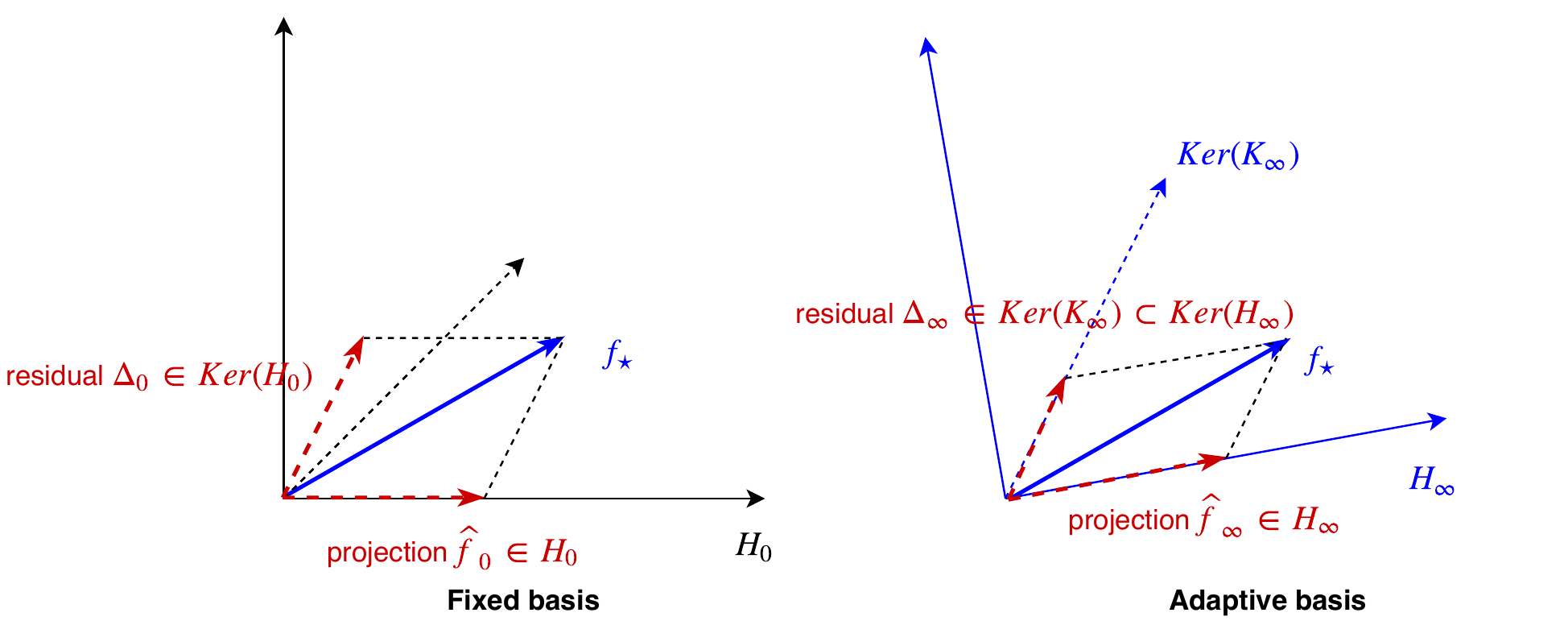}
            \caption{\small Illustration of \ref{thm:gap-decomposition}: fixed basis vs. adaptive learned basis. In classic statistics, one specifies the fixed function space/basis $H_0$ then decompose $f_*$ into the projection $\hat{f}_0$ and residual $\Delta_0 \in \text{Ker}(H_0)$. However, for GD on NN, one learns the adaptive basis $H_\infty$ that depends on $f_*$. Therefore, the residual $\Delta_\infty$ lies in a subspace of $\text{Ker}(H_\infty)$.}
            \label{fig:adaptive-basis}
        \end{figure}

    Before concluding this section, we remark that a similar result holds for the ERM problem \eqref{eq:erm}. As we shall discuss in the next section, the gap between $\cH_\infty$ and $\cK_\infty$ can be large, even for the ERM problem.


\section{Implications of the Adaptive Theory}
\label{sec:implication-of-theory}

In this section, we will discuss some direct implications of the adaptive kernel theory for neural networks established in this paper.

\paragraph{Example: Gap in Spaces $\cH_\infty$ and $\cK_\infty$.}
In Theorem~\ref{thm:gap-decomposition}, it is established that $\text{Ker}(\cK_\infty) \subset \text{Ker}(\cH_\infty)$. 
We now construct a concrete case to illustrate the potentially significant gap in these two spaces as follows.
Consider only one neuron with $m=1$, solving ERM problem~\eqref{eq:erm} with $n$ samples, and $\bx$ with dimension $d$. In this case, $\rho_\infty$ is supported on only one point, noted as $\Theta_\infty \in \mathbb{R}^d$. Denote $X \in \mathbb{R}^{n\times d}$ as the data matrix, one can show that
\begin{align*}
    H_{\infty}(X, X) = \underbrace{\sigma(X \Theta_\infty^T)}_{n\times 1} \underbrace{\sigma(X \Theta_\infty^T)^T}_{1\times n}
\end{align*}
has rank $1$. In contrast, 
\begin{align*}
    K_{\infty}(X, X) \succeq \underbrace{{\rm diag}(\mathbbm{1}_{X\Theta_\infty^T \geq 0}) X }_{n\times d}   \underbrace{ X^T {\rm diag}( \mathbbm{1}_{X\Theta_\infty^T \geq 0}) }_{d\times n} 
\end{align*}
can be of rank $d \wedge  |\{i: x_i^T \Theta_\infty \geq 0 \}|$. Hence the null space of $K_\infty$ is much smaller than that of $H_\infty$. The gap can be large for many other settings of $(n,m,d)$.

\paragraph{Connections to Min-norm Interpolation.}

The following result establishes the connections between the solution of gradient descent on neural networks (at local stationarity), and the kernel ridgeless regression \citep{belkin2018understand, liang2018just, hastie2019surprises} with an adaptive kernel $\widehat{H}^{\lambda}_{\infty}$. Empirical evidence on the similarity between the interpolation with kernels and neural networks was discovered in \cite{belkin2018understand}. The following proposition provides a novel way of studying the generalization property of neural networks via adaptive kernels.

\begin{proposition}[Interpolation: Connection to Kernel Ridgeless Regression]
    \label{prop:interpolation}
    Consider the gradient flow dynamics on all the weights of the neural network function $f_t(x) = \sum_{j=1}^m w_j(t) \sigma(x^T u_j(t))$, to solve the $\ell_2$-regularized ERM
    \begin{align*}
        \frac{1}{2n} \sum_{i=1}^n (y_i - f_t(x_i))^2 + \frac{\lambda}{2m} \sum_{j=1}^m \left[ w_j(t)^2 + \| u_j(t) \|^2 \right] \enspace.
    \end{align*}
    Consider only the bounded assumption on initialization that $|w_j^2(0) - \|u_j\|^2(0)| <\infty$ for all $1\leq j\leq m$.
    At stationarity, denote the signed measure as $\widehat{\rho}^{\lambda}_{\infty}$ and the corresponding adaptive kernel as $\widehat{H}_{\infty}^{\lambda}$.
    Then the neural network function at stationarity $\widehat{f}^{{\rm nn}, \lambda}_{\infty} (x)$ satisfies, 
    \begin{align*}
        \widehat{f}^{{\rm nn}, \lambda}_{\infty} (x) = \widehat{H}_{\infty}^{\lambda}(x, X) \left[ \frac{n}{m} \lambda \cdot I_n + \widehat{H}_{\infty}^{\lambda}(X, X) \right]^{-1} Y \enspace.
    \end{align*}
\end{proposition}

In the vanishing regularization $\lambda \rightarrow 0$ limit, the neural network function converges to the kernel ridgeless regression with the adaptive kernel, when $\widehat{H}_{\infty}(X, X) := \lim_{\lambda \rightarrow 0} \widehat{H}_{\infty}^{\lambda}$ exists,
\begin{align*}
     \lim_{\lambda \rightarrow 0} \widehat{f}^{{\rm nn}, \lambda}_{\infty} (x) = \widehat{H}_{\infty}(x, X) \widehat{H}_{\infty}(X, X)^{+} Y = \widehat{f}^{\rm rkhs}_{\infty}(x).
\end{align*}
Note that the generalization theory for the kernel ridgeless regression has been established \cite{liang2018just, hastie2019surprises}. Here the kernel $\widehat{H}_{\infty}(X, X)$ is data-adaptive (that adapts to $f_*$) learned along training, instead of being fixed and pre-specified.

\paragraph{Connections to Random Kitchen Sinks.}
Let us introduce two function spaces, with the base measure $\rho_0$ (fixed representation)
\begin{align*}
    \Gamma_2(\rho_0):= \left\{ f(x)~|~ f(x) = \int \sigma(x^T\Theta) w(\Theta) \rho_0(d\Theta), w \in L^2_{\rho_0} \right\}
\end{align*}
\begin{align*}
    \Gamma_1(\rho_0):= \left\{ f(x)~|~ f(x) = \int \sigma(x^T\Theta) w(\Theta) \rho_0(d\Theta), w \in L^1_{\rho_0} \right\}
\end{align*}

In random kitchen sinks studied in \cite{rahimi2008random, rahimi2009weighted}, by assuming $f_* \in \Gamma_2(\rho_0)$ that lies in the RKHS, the approximation error can be controlled by the existence of the following function with $\theta_j, j \in [m]$ i.i.d. sampled from $\rho_0$ 
\begin{align*}
    \widehat{f}(x) = \frac{1}{m} \sum_{j=1}^m \sigma(x^T\Theta_j) w(\Theta_j)  \in \Gamma_1(\rho_0), \text{but} ~\widehat{f}(x) \notin \Gamma_2(\rho_0) \enspace.
\end{align*}
Note that $\widehat{f}$ lies in a possibly much larger space $\Gamma_1(\rho_0)$ though the target only lies in $f_* \in \Gamma_2(\rho_0)$. Similarly for two-layer neural networks function $f_t(x)$ considered in \cite[Section 2.3]{bach2017breaking}, the RKHS space $\Gamma_2(\rho_0)$ can be more restrictive compared to $f_t\in \Gamma_1(\rho_0)$. 

In contrast, with the adaptive RKHS representation $\cH_\infty$, we have shown that
\begin{align*}
    f_\infty(x) \in \Gamma_1(|\rho_\infty|), \text{and} ~f_\infty(x) \in \Gamma_2(|\rho_\infty|) \enspace.
\end{align*}

The extreme case of fully adaptive function space $\Gamma_2(|\rho_*|)$ is defined with $\rho_*$ tailored for $f_*$, $
    f_* = \int \sigma(x^T\Theta) \rho_*(d\Theta)$.
The adaptive representation learned by neural networks can be viewed as in between the fixed and the fully adaptive representation.

\paragraph{Adaptive Generalization Theory.}

Now we attempt to provide a new decomposition to study the generalization of NN via adaptive kernels. Recall we have shown that
$\widehat{f}^{\rm rkhs}_{\infty}(x) = \lim_{\lambda \rightarrow 0} \widehat{f}^{{\rm nn}, \lambda}_{\infty} (x) = \widehat{H}_{\infty}(x, X) \widehat{H}_{\infty}(X, X)^{+} Y,$
where
$\widehat{H}_{\infty}(x, \tilde{x}) :=   \int \sigma(x^T \Theta) \sigma(\tilde{x}^T \Theta) \widehat{\rho}^{(n, m)}_\infty(d\Theta).$
Define the population limit
$\rho^{(m)}_{\infty}(d\Theta) := \lim_{n\rightarrow \infty } \widehat{\rho}^{(n, m)}_\infty $ and  $H_{\infty}(x, \tilde{x}) := \int \sigma(x^T \Theta) \sigma(\tilde{x}^T \Theta) \rho^{(m)}_\infty (d\Theta)$. Denote the ridgeless regression with the population adaptive kernel $H_\infty$,
\begin{align*}
    f^{\rm rkhs}_{\infty}(x) = H_{\infty}(x, X) H_{\infty}(X, X)^{+} Y.
\end{align*}
Assume $(\by - f_*(\bx))^2 \leq \sigma^2$ a.s. (can be relaxed). One can derive the following decomposition for generalization.
\begin{proposition}[Adaptive Generalization]
    \label{prop:adaptive-generalization}
    \begin{align*}
        \| \lim_{\lambda \rightarrow 0} \widehat{f}^{{\rm nn}, \lambda}_{\infty}- f_*\|^2_{\mu} & \precsim \underbrace{\| \widehat{f}^{\rm rkhs}_{\infty} - f^{\rm rkhs}_{\infty}  \|_\mu^2}_{\text{adaptive representation error}}  + \underbrace{\| f_\infty - f_* \|_{\mu}^2 }_{\text{adaptive approximation error}} \\
        & \quad \quad + (n \| f_\infty - f_* \|_{\hat{\mu}}^2 +\sigma^2) \underbrace{\E_{\bx \sim \mu} \| H_\infty(X, X)^{-1} H_\infty(X, \bx)\|^2}_{\text{adaptive variance}} \\
        & \quad \quad + \underbrace{\| H_{\infty}(x, X) H_\infty(X, X)^{-1} f_\infty(X) -  f_\infty(x) \|_\mu^2}_{\text{adaptive bias}}
    \end{align*}
\end{proposition}

Note this result holds without requiring global optimization guarantees. The first term is the representation error, which corresponds to the closeness of the adaptive RKHS $\widehat{\cH}_\infty$ (using empirical distribution) and $\cH_\infty$ (using population distribution). The second term is the adaptive approximation error studied in the current paper. The third and fourth terms are the variance and bias expressions studied in \cite{liang2018just, hastie2019surprises, rakhlin2018consistency}, as if assuming the actual function lies in $\cH_\infty$. This decomposition suggests the possibility of studying generalization without explicit global understanding of the optimization, and providing rates that adapts to $f_*$ without structural assumptions.


\section{Time-varying Kernels and Evolution}    
\label{sec:kernel}

In this section, we lay out the mathematical details on the time-varying kernels and the evolution of the signed measure $\rho_t$ supporting the main results. In the meantime, we will discuss in depth the relevant literature motivating our proof ideas.

First, we describe the motivation behind the dynamic RKHS $\cK_t$, and the GD kernel induced by the gradient descent dynamics. Extensions to multi-layer perceptrons is in Sec.~\ref{sec:extension-MLP}.
     \begin{lemma}[Dynamic kernel of finite neurons GD]
        \label{lem:dynamic-kernel}
        Consider the approximation problem \eqref{appr_prob} with a neural network function \eqref{nn}, and the training process \eqref{eq:training_2} with population distribution. Let $\Delta_t(x) = f_*(x)-f_t(x)$ be the residual. Define the time-varying kernel $K_t(\cdot, \cdot): \mathcal{X} \times \mathcal{X} \rightarrow \mathbb{R}$,
        \begin{align}\label{def:gd-kernel}
        K_t(x ,\tilde x) &=  \sum_{j=1}^m \bigg[\sigma(x^Tu_j(t))\sigma(\tilde x^Tu_j(t))  + w_j(t)^2\mathbbm{1}_{x^Tu_j(t)\geq0}\mathbbm{1}_{\tilde x^Tu_j(t)\geq0}x^T\tilde x \bigg].
        \end{align} 
        Then the residual $\Delta_t$ driven by the GD dynamics satisfies,
        \begin{align}
            \frac{d\mathbf{E}_{ \bd{ x}}\left[\frac{1}{2}\Delta_t(\bd{x})^2\right]}{dt} =  -\mathbf{E}_{\bd{x}, \bd{\tilde x}}\left[ \Delta_t(\bd{x})K_t(\bd{x}, \bd{\tilde x})\Delta_t(\bd{\tilde x})\right] .
        \end{align}
    \end{lemma}

When running GD to solve the empirical risk minimization (ERM), the dynamics of the finite-dimensional sample residual $\| \Delta_t \|_{\hat{\mu}}^2 $ has been established in \cite{jacot2018neural,du2018gradient}. Here we generalize the result to optimize the weights of both layers, and to solve the infinite-dimensional population approximation problem rather than the empirical risk minimization problem.
For a general loss function $\ell(y, f)$ with curvature (say, logistic loss), similar results hold under slightly stronger conditions.
\begin{corollary}
            \label{cor:general-loss}
            Consider a general loss function $\ell(y, f)$ that is $\alpha$-strongly convex in the second argument $f$, with $K_t$ defined in \eqref{def:gd-kernel}. Assume in addition $\frac{1}{n}K_t(X, X) \in \mathbb{R}^{n \times n}$ has smallest eigenvalue $\lambda_t > 0$. Define $\Delta_t(x_i) := \frac{\partial \ell(y_i,f_t(x_i))}{\partial f}$, then we have for all $f_* : \mathbb{R}^d \rightarrow \mathbb{R}$,  
             \begin{align*}
                         \frac{d \widehat{\E} \left[\ell(\by, f_t(\bx))\right]}{dt} =  -\widehat{\E}_{\bd{x}, \bd{\tilde x}}\left[ \Delta_t(\bd{x})K_t(\bd{x}, \bd{\tilde x})\Delta_t(\bd{\tilde x})\right] \leq - 2\alpha\lambda_t \cdot \widehat{\E} \left[\ell(\by, f_t(\bx)) - \ell(\by, f_*(\bx))\right] \enspace.
              \end{align*} 
\end{corollary}

\subsection{Initialization, Rescaling and $K_0$}
\label{sec:K0}

Now we describe the initialization and rescaling schemes used in the main theorems. 
Rewrite \eqref{appr_prob} according to the signs of the second layer weights
\begin{align*}
    f_t(x) := \sum_{j=1}^{m_{+}} w_{+,j}(t)\sigma(x^T u_{+,j}(t)) + \sum_{j=1}^{m_{-}}w_{-,j}(t)\sigma(x^T u_{-,j}(t)).
\end{align*}

\paragraph{Initialization.} 
We consider the ``infinitesimal'' initialization drawn $i.i.d.$ from two probability measures $\rho_{+, 0}$ and $\rho_{-, 0}$ that do not depend on $m$:
\begin{align}
    u_{+,j}(0) = \frac{1}{\sqrt{m}} \Theta_{+, j}  ~\text{where}~ \Theta_{+, j} \sim \rho_{+, 0} ~ , ~
    u_{-,j}(0) = \frac{1}{\sqrt{m}} \Theta_{-, j}  ~\text{where}~ \Theta_{-, j} \sim \rho_{-, 0} ~.
\end{align} 
Here $m = m_+ + m_-$ with $m_+ \asymp m_-$.
The $1/\sqrt{m}$ rescaling factor turns out to be crucial when defining the infinite neurons limit for the evolution of signed measures. Remark that such initialization is w.l.o.g., and accounts for the infinitesimal nature used in practice when the number of neurons grows. For the second layer weights, we impose the ``balanced condition'' motivated by \cite{maennel2018gradient},
\begin{align}
    w_{+,j}(0) = \|u_{+,j}(0)\|\geq 0 ~,~ w_{-,j}(0) = -\|u_{-,j}(0)\| \leq 0.
\end{align}
It turns out that with such initialization, the balanced condition holds throughout the training process induced by gradient flow, which is useful for the main theorems. Interestingly, in the proof of Proposition~\ref{prop:interpolation}, we show that such balanced condition always holds at stationarity when training neural networks with $\ell_2$ regularization, even for unbalanced initialization.
\begin{proposition}[Balanced condition]\label{lem:balance}
For $u_{+, j}(t)$, $u_{-, j}(t)$, $w_{+, j}(t)$ and $w_{-, j}(t)$, and the initialization specified above, at any time $t$, we have 
\begin{align*}
    w_{+, j}(t) = \|u_{+, j}(t)\|, \quad
    w_{-, j}(t) = -\|u_{-, j}(t)\|.
\end{align*}
\end{proposition}

\paragraph{Rescaling.}
To prepare for the distribution dynamic theory in the next section, we introduce a parameter rescaling with the $\sqrt{m}$ factor. Let $\theta_{+,j}(t) = \sqrt{m}w_{+,j}(t)$ and $\theta_{-,j}(t)=\sqrt{m}w_{-,j}(t)$, also define $\Theta_{+,j}(t) = \sqrt{m}u_{+,j}(t)$ and $\Theta_{-,j}(t)= \sqrt{m}u_{-,j}(t)$ sampled from $\rho_{+, 0}$ and $\rho_{-, 0}$ at $t = 0$. Under this representation,
\begin{align}
    \label{eq:rescaled-representation}
    f_t(x) = \frac{1}{m}\sum_{j=1}^{m_{+}} \theta_{+,j}(t)\sigma(x^T\Theta_{+,j}(t))+ \frac{1}{m}\sum_{j=1}^{m} \theta_{-,j}(t)\sigma(x^T\Theta_{-,j}(t)).
\end{align}
By the positive homogeneity of ReLU, we have the corresponding dynamics on the rescaled parameters,
 \begin{align}
    \frac{d\theta_{\cdot,j}}{dt} &= \sqrt{m} \frac{d w_{\cdot,j}}{dt}=  -\sqrt{m}\mathbf{E}_{\bd{z}}\left[\frac{\partial \ell(\by,f(\bx))}{\partial f}\sigma(\bx^Tu_{\cdot,j})\right] = -\mathbf{E}_{\bd{z}}\left[\frac{\partial \ell(\by,f(\bx))}{\partial f}\sigma(\bx^T\Theta_{\cdot,j})\right], \label{def:training-rescale_1}\\
    \frac{d\Theta_{\cdot,j}}{dt} &= \sqrt{m} \frac{d u_{\cdot,j}}{dt}= -\sqrt{m}\mathbf{E}_{\bd{z}}\left[\frac{\partial \ell(\by,f(\bx))}{\partial f}w_{\cdot, j} \mathbbm{1}_{\bx^Tu_{\cdot, j}\geq0}\bx\right] = -\mathbf{E}_{\bd{z}}\left[\frac{\partial \ell(\by,f(\bx))}{\partial f}\theta_{\cdot, j}\mathbbm{1}_{\bx^T\Theta_{\cdot, j}\geq0}\bx\right].\label{def:training-rescale_4}
\end{align} 
Define at time $t$
\begin{align}
    \label{eq:sign-measure}
    \rho_{+, t} := \frac{1}{m} \sum_{j=1}^{m_{+}} \delta_{\Theta_{+, j}(t)}, ~~ \rho_{-, t} := \frac{1}{m} \sum_{j=1}^{m_{-}} \delta_{\Theta_{-, j}(t)}
\end{align}
as the empirical distribution over neurons on the parameter space $\Theta$. The $\rho_{+, t}$ and $\rho_{-, t}$ converge weakly to proper distributions in the infinite neurons limit $m\rightarrow \infty$, see e.g. \cite{bach2017breaking, mei2018mean}.
Through the balanced condition in Proposition~\ref{lem:balance} and Proposition~\ref{lem:no-change-sign}, we know (by substituting $\theta_j$ by $\|\Theta_j\|$ )
\begin{align}\label{def:para-nn}
     f_t(x) = \int \|\Theta\|\sigma(x^T\Theta)\rho_t(d\Theta), ~\text{where the signed measure $\rho_t := \rho_{+, t} - \rho_{-, t}$}.
\end{align}
The above motivates the study of the RKHS $\cH_t$ as in Theorem~\ref{thm:proj-solution}, with the kernel
\begin{align}
    \label{eq:h}
    H_t(x, \tilde{x}) = \int \| \Theta \|^2 \sigma(x^T\Theta) \sigma(\tilde x^T\Theta) |\rho_t|(d \Theta).
\end{align}

To conclude this section, we provide the explicit formula for the initial kernel matrix $K_0$ under such infinitesimal random initialization.
Specifically, consider the initialization with $w_j$ being $\pm 1/\sqrt{m}$ with equal chance and $u_i \sim N(\mathbf{0},1/m \cdot \bd{I}_d)$ $i.i.d.$ sampled. The initial kernel $K_0$ has the following expression, in the infinite neurons limit.
\begin{lemma}[Fixed Kernel]
    \label{lem:K_0}
With initialization specified above, consider w.l.o.g. $\|x\| = \|\tilde  x\|=1$, and denote $\Theta \sim \pi$ as the isotropic Gaussian $N(\mathbf{0}, \mathbf{I}_d)$. By the strong law of large number, we have almost surely, 
\begin{align*}
    \lim_{m \rightarrow \infty} K_0(x,\tilde x) &=   \mathbf{E}_{\Theta \sim \pi}\left[ \sigma(x^T\Theta)\sigma(\tilde x^T\Theta) + \mathbbm{1}_{x^T \Theta>0}\mathbbm{1}_{\tilde x^T \Theta>0}x^T\tilde x\right]\\
    &=  \left[ \frac{\pi - \arccos(t)}{\pi} t + \frac{\sqrt{1-t^2}}{2\pi}  \right], \quad\text{where $t = x^T\tilde  x$}.
\end{align*}
\end{lemma}
Much known results \citep{bengio2006convex, rahimi2008random, bach2017breaking, cho2009kernel, daniely2016toward} on the connection between RKHS and two-layer NN focus on some fixed kernel, such as $K_0$. To instantiate useful statistical rates, one requires $f_*$ to lie in the corresponding pre-specified RKHS $\mathcal{K}_0$, which is non-verifiable in practice. In contrast, the dynamic kernel is less studied. We will establish a dynamic and adaptive kernel theory defined by GD, without making any structural assumptions on $f_*$ other than $f_* \in L^2_\mu$.

\subsection{Evolution of $\rho_t$}
\label{sec:pde-char}

In this section, we derive the evolution of the signed measure $\rho_t$ defined by the neurons at the training $t$, which in turn determines the dynamic kernel $K_t$ defined in \eqref{def:gd-kernel}. To generalize the result to the case of infinite neurons, we follow and borrow tools from the mean-field characterization \citep{mei2018mean, rotskoff2018neural, jordan1998variational}. The rescaling described in the previous section proves handy when defining such infinite neurons limit. 
We define the velocity field driven by the regression task and the interaction among neurons,
\begin{align}
    \label{eq:velocity} 
    V(\Theta) = \mathbf{E}[\bd{y}\sigma(\bx^T\Theta)], \quad 
    U(\Theta,\tilde{\Theta}) = -\mathbf{E}[\sigma(\bx^T\Theta)\sigma(\bx^T\tilde{\Theta})].
\end{align}
The following theorem casts the training process as distribution dynamics on $\rho_{+, t},\rho_{-, t}$. 

\begin{lemma}[Dynamic Kernel and Evolution]
    \label{lem:pdf-sign-measure}
    Consider the approximation problem \eqref{appr_prob}, and the gradient flow as the training dynamic \eqref{eq:training_2}. For $\rho_{+, t}$, $\rho_{-, t}$ and $\rho_{t}$ defined in \eqref{eq:sign-measure} with possibly infinite neurons, we have the following PDE characterization on distribution dynamics of $\rho_{+, t},\rho_{-, t}$
    \begin{align}\label{def:training-infinite}
        \partial_t\rho_{+,t}(\Theta) = - \nabla_{\Theta} \cdot\left[\rho_{+,t}(\Theta) \cdot \|\Theta\| \left( \nabla _{\Theta}V(\Theta)+ \nabla_{\Theta}\int U(\Theta,\tilde\Theta)\|\tilde\Theta\|\rho_t(d \tilde\Theta)\right)\right], \nonumber\\
        \partial_t\rho_{-,t}(\Theta) = \nabla_{\Theta} \cdot\left[ \rho_{-,t}(\Theta)\cdot \|\Theta\| \left( \nabla _{\Theta}V(\Theta)+ \nabla_{\Theta}\int U(\Theta,\tilde\Theta)\|\tilde\Theta\|\rho_t(d \tilde\Theta)\right)\right].
    \end{align} 
    Moreover, the GD kernel $K_t$ is defined as
    \begin{align}
        \label{eq:k}
        K_{t}(x,\tilde x)= \int \left(\|\Theta\|^2\mathbbm{1}_{x^T\Theta\geq0}\mathbbm{1}_{\tilde x^T\Theta\geq0}x^T \tilde x + \sigma(x^T\Theta)\sigma(\tilde x^T\Theta) \right) |\rho_t|(d\Theta).
    \end{align}
\end{lemma}

\begin{remark}
    \rm 
    As in \cite{mei2018mean, rotskoff2018neural}, let's first show that in the infinite neurons limit $m \rightarrow \infty$, $\rho_{+, t}, \rho_{-, t}$ are properly defined, with Eqn.~\eqref{def:training-infinite} characterizing the distribution dynamics. For simplicity, we assume the initialization $\rho_{+,0},\rho_{-,0}$ is with bounded support. Add the superscript $m$, $\rho_{+,t}^{(m)}, \rho_{-,t}^{(m)},\rho_{t}^{(m)}$ to \eqref{eq:sign-measure} to indicate their dependence on $m$. Consider that $\nabla_\Theta V(\Theta)$, $\nabla_{\Theta}U(\Theta,\tilde \Theta)$ in \eqref{eq:velocity} are bounded and uniform Lipchitz continuous as in \cite[A3]{mei2018mean}. With the same proof as in \cite[Theorem 3]{mei2018mean}, one can show that with $m\rightarrow \infty$, the initial distribution $\rho^{(m)}_0 \xrightarrow{d}  \rho_0 = \rho_{+,0} - \rho_{-,0}$ by law of large number. And by the solution's continuity w.r.t. the initial value, we have $\rho_t^{(m)} \xrightarrow{d} \rho_t$ as $m\rightarrow \infty$ well defined, for any fixed $t$.
    
    Note that our problem setting is slightly different from that in \cite{mei2018mean}, where the authors consider the NN with fixed second layer weights to be $1/m$. 
 We reiterate that the re-parameterization via $\theta$ and $\Theta$ is crucial: (1) weights on both layers are optimized following the gradient flow; (2) infinitesimal random initialization is employed in practice. 
    In the setting of \cite[Eqn. (3)]{mei2018mean}, the training process is slightly different from the vanilla GD on weights, with an additional $m$ factor in the velocity term. This subtlety is also mentioned in \cite{rotskoff2018neural}. In short, the rescaling looks at the dynamics where $\Theta$'s are on the invariant scale as $m \rightarrow \infty$ for any fixed effective time $t$ (that does not depend on $m$).
    Here we analyze the exact gradient flow on the two-layer weights, with infinitesimal random initialization as in practice, resulting in a different velocity field \eqref{eq:velocity} compared to that in \cite{mei2018mean}. 
    
\end{remark}

The proof of Theorem~\ref{thm:proj-solution} makes use of \eqref{def:para-nn}-\eqref{eq:h} and the stationary condition implied by Lemma~\ref{lem:pdf-sign-measure}. The balanced condition is crucial in both Theorem~\ref{thm:proj-solution} and Proposition~\ref{prop:interpolation}.
The details of the proof are deferred to Section~\ref{sec:pf}.

\subsection{Two RKHS: $\cK_\infty$ and $\cH_\infty$}


In this section we compare the two adaptive RKHS appeared $\cK_\infty$ in \eqref{eq:k}, and $\cH_\infty$ in \eqref{eq:h}. The comparison will lead to the proof of Theorem~\ref{thm:gap-decomposition}. We start with generalizing Lemma~\ref{lem:dynamic-kernel} with the possibly infinite neurons case via the distribution dynamics in \eqref{def:training-infinite}.
    
    \begin{corollary}
        \label{cor:signed-kernel}
       Consider the same setting as in Lemma~\ref{lem:dynamic-kernel} with possibly infinite neurons NN \eqref{def:para-nn}, and the training process \eqref{def:training-infinite}. Define the time-varying kernel matrix $K_t(\cdot, \cdot): \mathcal{X} \times \mathcal{X} \rightarrow \mathbb{R}$, with the signed measure $\rho_t$ follows \eqref{def:training-infinite}
    \begin{align}\label{def:inf-gd-kernel}
        K_{t}(x,\tilde x) &= \int \left(\|\Theta\|^2\mathbbm{1}_{x^T\Theta\geq0}\mathbbm{1}_{\tilde x^T\Theta\geq0}x^T \tilde x + \sigma(x^T\Theta)\sigma(\tilde x^T\Theta) \right) |\rho_t|(d\Theta) \\
        & =: K_t^{(0)}(x,\tilde x) + K_t^{(1)}(x,\tilde x).
    \end{align}
        Then we still have 
    $
        d\mathbf{E}_{ \bd{ x}}\left[\frac{1}{2}\Delta_t(\bd{x})^2\right] / dt =  -\mathbf{E}_{\bd{x}, \bd{\tilde x}}\left[ \Delta_t(\bd{x})K_t(\bd{x}, \bd{\tilde x})\Delta_t(\bd{\tilde x})\right].
   $ 
    \end{corollary}

   It turns out that the kernels $K_\infty$ and $H_\infty$, defined in \eqref{eq:rkhs-gd} and \eqref{def:rkhs-kernel} respectively, satisfy the following inclusion property. 
    \begin{proposition}\label{prop:kernel-compare}
        Consider the training process reaches any stationarity $\rho_\infty = \rho_{+,\infty} - \rho_{-,\infty}$ with compact support within radius $D$ and finite total variation. We have 
        \begin{align}\label{ieq:pos-compare}
            K_{\infty} \succeq K_{\infty}^{(0)} \succeq K_{\infty}^{(1)} \succeq \frac{1}{D^2}H_\infty,
        \end{align}
        with $K_{\infty}^{(0)}, K_{\infty}^{(1)}$ defined in \eqref{def:inf-gd-kernel}. 
        Combining with the fact that $H_\infty \neq K_\infty$ implies
        \begin{align*}
            \text{Ker}( \mathcal{K}_\infty ) \subset  \text{Ker}( \mathcal{H}_\infty).
        \end{align*}
    \end{proposition}

    The proof of Theorem~\ref{thm:gap-decomposition} uses the following fact: when reaching stationarity, due to the ODE defined by GD in Lemma~\ref{lem:dynamic-kernel}, the residual must satisfy
    \begin{align}\label{eq:null-gd}
        \Delta_\infty(x) = f_*(x) - f_{\infty}(x) \in \text{Ker}(\mathcal{K}_{\infty}).
    \end{align}
    The proof of Proposition~\ref{prop:kernel-compare} and Theorem~\ref{thm:gap-decomposition} are deferred to Section~\ref{sec:pf}.


\section{Experiments}

We run experiments to illustrate the spectral decay of the dynamic kernels defined in $K_t$ over time $t$. The exercise is to quantitatively showcase that during neural network training, one does learn the data-adaptive representation, which is task-specific depending on the true complexity of $f_*$. The training process is the same as the one we theoretically analyze: vanilla gradient descent on a two-layer NN of $m$ neurons, with infinitesimal random initialization scales as $1/\sqrt{m}$.

The first experiment is a synthetic exercise with well-specified models. We generate $\{x_i\}_{i=1}^{50}$ from isotropic Gaussian in $\mathbb{R}^5$, and $y_i = f_*(x_i) = \sum_{j=1}^J w^*_j\sigma(x_i^Tu^*_j)$  with different $J$. In other words, we choose different target $f_*$ (task complexity) by varying $J$. We select $m=500$ in our experiment. The top $80\%$ of the sorted eigenvalues of the kernel matrix $K_t$ along the GD training process are shown in Fig. \ref{fig:kernel-true}. The $x$-axis is the index of eigenvalues in descending order, and the $y$-axis is the logarithmic values of the corresponding eigenvalues. Different color indicates the spectral decay of the $K_t$ at different training time $t$. The eigenvalue-decays stabilize over time $t$ means that the training process approaches stationarity. As we can see with $f_*$ belongs to the NN family, the eigenvalues of the kernel matrix, in general, become larger during the training process. For a more complicated target function, it takes longer to reach stationarity.

    \begin{figure}[ht!]
    \centering
    \begin{subfigure}[b]{0.4\textwidth}
            \includegraphics[width = \textwidth]{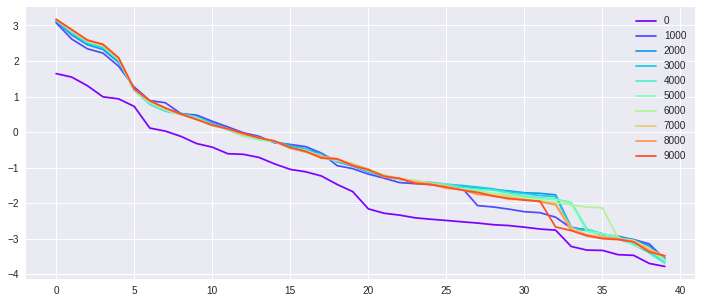}
            \caption{$J=2$}
            \label{fig:1a}
        \end{subfigure}
        ~
        \begin{subfigure}[b]{0.4\textwidth}
            \includegraphics[width = \textwidth]{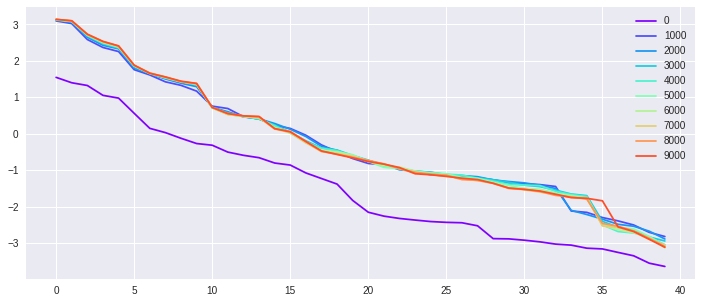}
            \caption{$J=4$}
            \label{fig:1b}
    \end{subfigure}
    
    \begin{subfigure}[b]{0.4\textwidth}
            \includegraphics[width = \textwidth]{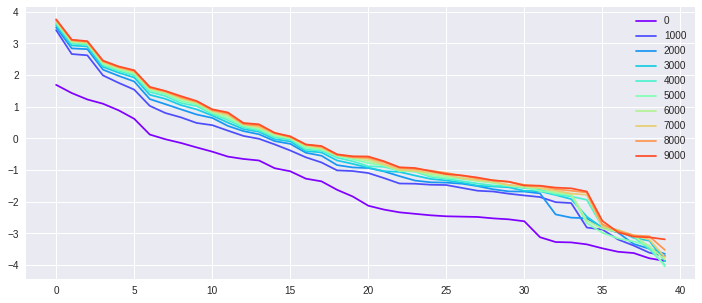}
            \caption{$J=8$}
            \label{fig:1c}
        \end{subfigure}
        ~
        \begin{subfigure}[b]{0.4\textwidth}
            \includegraphics[width = \textwidth]{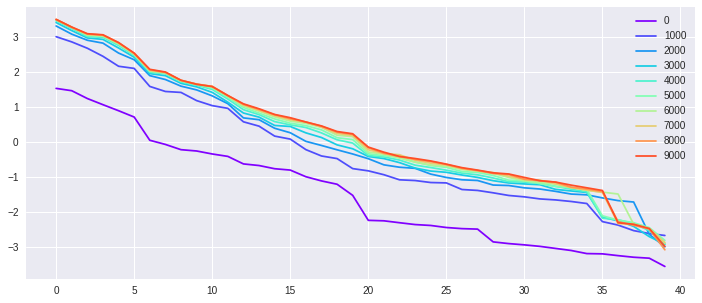}
            \caption{$J=16$}
            \label{fig:1d}
    \end{subfigure}

    \caption{\small Log of the sorted top $80\%$ eigenvalues of kernel matrix along training with different $f_*$}
    \label{fig:kernel-true}
    \end{figure}

The second experiment is another synthetic test on fitting random labels. We generate $\{x_i\}_{i=1}^{50}$ from isotropic Gaussian in $\mathbb{R}^5$, as $y_i$ takes $\pm1$ with equal chance. We select $m=200,500$, and $n=50,200$ to investigate those parameters' influence on the kernel $K_t$. 
We want to point out two observations. First, fixed $n$, we investigate over-parametrized models ($m=200, 500$ large). Shown from Fig.~\ref{fig:kernel-noise} along the row, the kernels for different $m$'s behave much alike. In other words, in the infinite neurons limit, the kernel will stabilize. Second, fixed $m$, we vary the number of samples $n$, to simulate different interpolation hardness. As seen from Fig.~\ref{fig:kernel-noise} along the column, the kernels and the convergence over time are distinct, reflecting the different difficulty of the interpolation. 

The third experiment (Fig.~\ref{fig:kernel-mnist}) is regression using the MNIST dataset with different sample size $n=50,200$. We hope to investigate the influence of sample size on the kernel matrix along the training process. For a larger sample size $N$, it takes longer to reach stationarity.

    \begin{figure}[ht!]
    \centering
    
    \begin{subfigure}[b]{0.4\textwidth}
            \includegraphics[width = \textwidth]{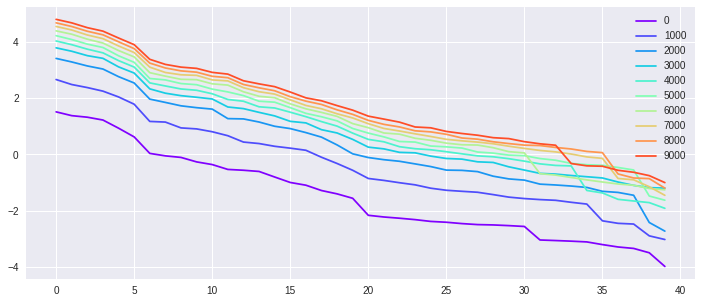}
            \caption{$N=50, m=200$}
            \label{fig:1a}
        \end{subfigure}
        ~
        \begin{subfigure}[b]{0.4\textwidth}
            \includegraphics[width = \textwidth]{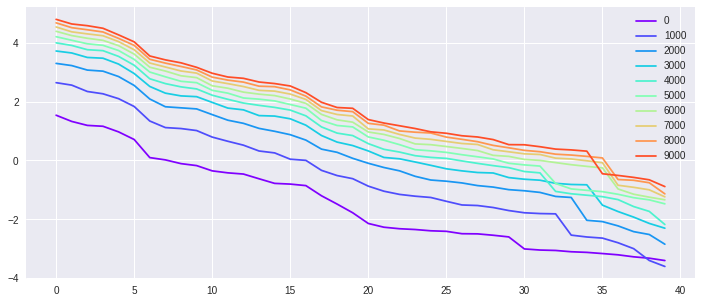}
            \caption{$N=50, m=500$}
            \label{fig:1b}
        \end{subfigure}
    
        \begin{subfigure}[b]{0.4\textwidth}
            \includegraphics[width = \textwidth]{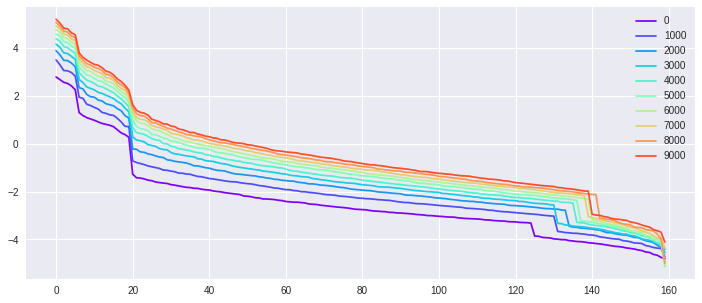}
            \caption{$N=200, m=200$}
            \label{fig:1c}
        \end{subfigure}
        ~
        \begin{subfigure}[b]{0.4\textwidth}
            \includegraphics[width = \textwidth]{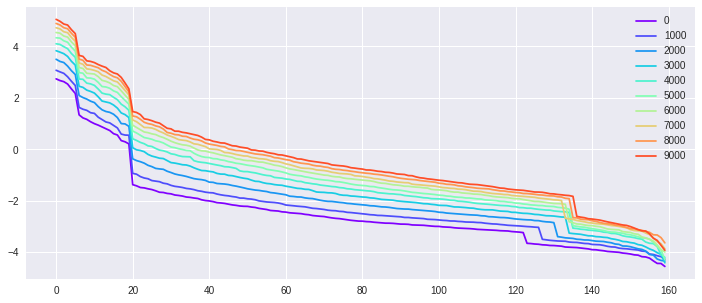}
            \caption{$N=200, m=500$}
            \label{fig:1d}
    \end{subfigure}
    
    \caption{\small Log of the sorted top $80\%$ eigenvalues of kernel matrix along training with random labels.}
    \label{fig:kernel-noise}
    \end{figure}

     \begin{figure}[ht!]
     \centering
     \begin{subfigure}[b]{0.3\textwidth}
             \centering
              \includegraphics[width = 0.8\textwidth]{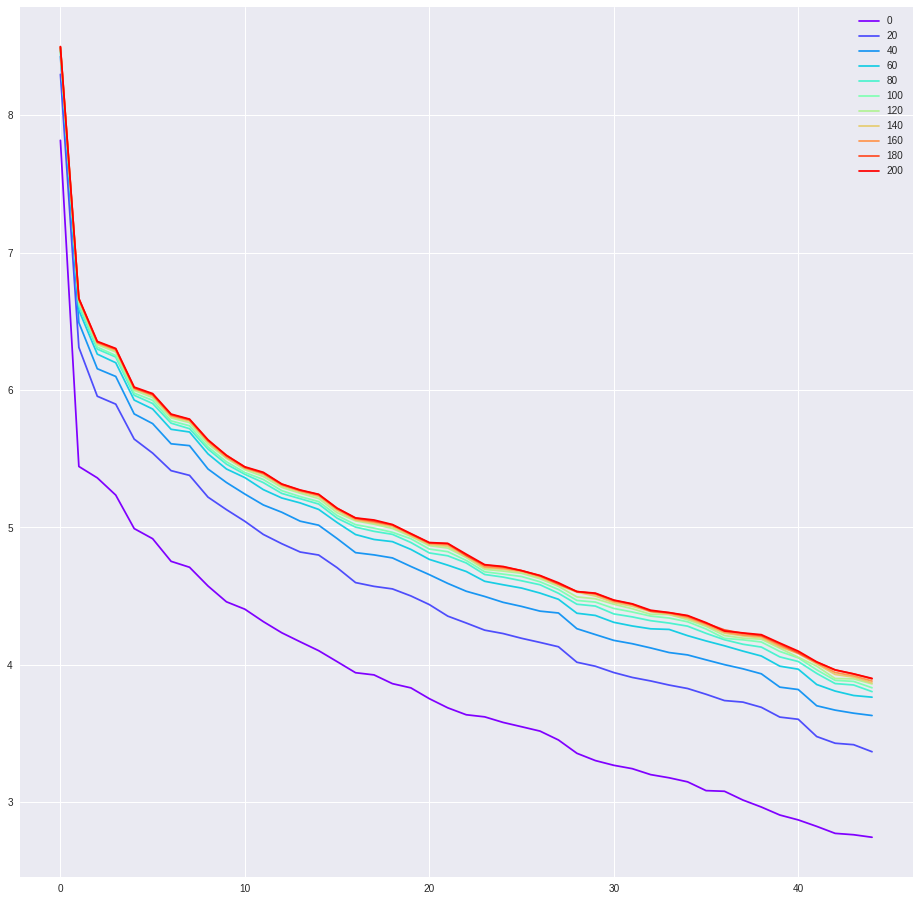}
              \caption{$N=50$}
              \label{fig:1a}
          \end{subfigure}
          ~
          \begin{subfigure}[b]{0.3\textwidth}
            \centering
              \includegraphics[width = 0.8\textwidth]{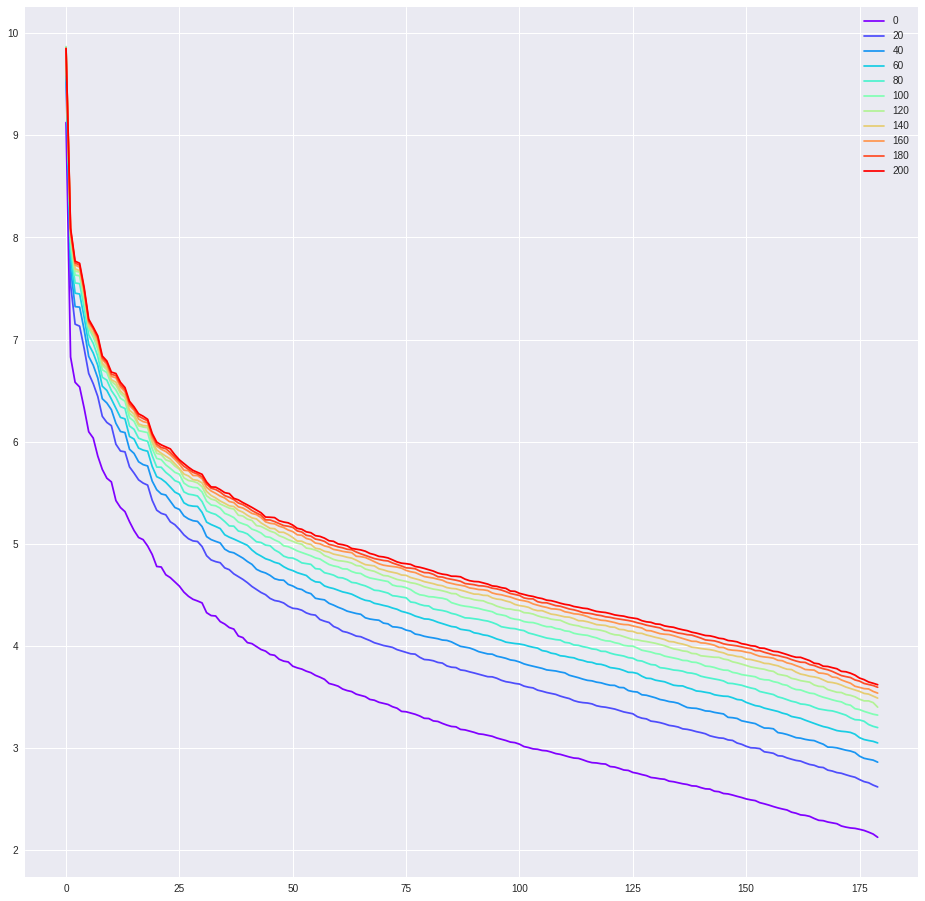}
              \caption{$N=200$}
              \label{fig:1b}
     \end{subfigure}
     \caption{\small Log of sorted top $90\%$ eigenvalues of kernel matrix along training process for mnist}
     \label{fig:kernel-mnist}
     \end{figure}

\section{Main Proofs}
\label{sec:pf}

    \begin{proof}[Proof of Theorem \ref{thm:proj-solution}]

        From the definition, we have $\mathcal{T}^* p\in\mathcal{H}_\infty$ for any $p \in L^2_{|\rho_\infty|}$, and $\mathcal{T}^*$ is a surjective mapping. 
         Suppose that $\widehat{g} \in \cH_\infty$ is a minimizer of \eqref{eq:rkhs-proj}, then we claim that for any $p \in L^2_{|\rho_\infty|}$, one must have
        \begin{align}
            \label{eq:normal-equation}
            \langle f_* - \widehat{g}, \mathcal{T}^* p \rangle_{\mu} = 0, ~~ \forall p \in L^2_{|\rho_\infty|}.
        \end{align}
        This claim can be seen from the following argument. Suppose not, then for $p$ that violates the above, construct
        \begin{align*}
            \widehat{g}_{\epsilon} = \widehat{g} + \epsilon \mathcal{T}^* p \in \mathcal{H}_\infty,
        \end{align*}
        we know
        \begin{align}
            \|f_* - \widehat{g}_{\epsilon} \|_{\mu}^2 = \|f_* - \widehat{g} \|_{\mu}^2 - 2\epsilon \langle f_* - \widehat{g}, \mathcal{T}^* p \rangle_{\mu} + \epsilon^2  \| \mathcal{T}^* p \|_{\mu}^2 .
        \end{align}
        For $\epsilon$ with the same sign as $\langle f_* - \widehat{g}, \mathcal{T}^* p \rangle_{\mu} \neq 0$ and small enough, one can see that $\|f_* - \widehat{g}_{\epsilon} \|_{\mu}^2 < \|f_* - \widehat{g} \|_{\mu}^2$ which validates that $\widehat{g}$ is a minimizer. 
        From the same argument, one can see that $\widehat{g}$ is a minimizer if and only if \eqref{eq:normal-equation} holds, in other words,
        \begin{align}
            \langle \mathcal{T} (f_* - \widehat{g}),  p \rangle_{|\rho_\infty|} = \langle f_* - \widehat{g}, \mathcal{T}^* p \rangle_{\mu} = 0
        \end{align}

        From PDE characterization \eqref{def:training-infinite} with ReLU activation, one knows that
        \begin{align*}
            V(\Theta) &= \mathbf{E}[\bd{y}\sigma(\bx^T\Theta)] = \E [ f_*(\bx)\sigma(\bx^T\Theta) ]\\
            U(\Theta,\tilde \Theta) &= -\mathbf{E}[\sigma(\bx^T\Theta)\sigma(\bx^T\tilde{\Theta})],
        \end{align*}
        and the expression for the velocity field
        \begin{align*}
            &\|\Theta\| \left( \nabla_{\Theta}V(\Theta)+ \nabla_{\Theta}\int U(\Theta,\tilde\Theta)\|\tilde\Theta\|\rho_t(d \tilde\Theta)\right) \\
            &= \|\Theta\| \left(\int f_*(x) x \mathbbm{1}_{x^T \Theta > 0} \mu(d x) - \int \int x \mathbbm{1}_{x^T \Theta > 0} \sigma(x^T \tilde\Theta) \| \tilde\Theta \| \rho_\infty(d\tilde\Theta)\mu(dx)  \right).
        \end{align*}
        We know that any stationary point $\left( \rho_{+,\infty}, \rho_{-,\infty} \right)$ has the following property \citep{mei2018mean}:
    \begin{align}\label{eq:fix-signed-point}
        \text{supp}(\rho_\infty) \subseteq \left\{ \Theta:  \int f_*(x) x \mathbbm{1}_{x^T \Theta > 0} \mu(d x) = \int \int x \mathbbm{1}_{x^T \Theta > 0} \sigma(x^T \tilde\Theta) \| \tilde\Theta \| \rho_\infty(d\tilde\Theta)\mu(dx) \right\}.
    \end{align}
    Multiplying both sides by $\|\Theta\|\Theta^T$ and recall the property of ReLU, the above condition implies that for all $\Theta \in \text{supp}(\rho_\infty)$, we have
    \begin{align}\label{eq:rkhs-gd-condition}
        \int f_*(x) \| \Theta \| \sigma(x^T \Theta) \mu(d x) = \int \int  \| \Theta \| \sigma(x^T \Theta) \sigma(x^T \tilde\Theta) \| \tilde\Theta \| \rho_\infty(d \tilde\Theta) \mu(dx).
    \end{align}
    One can see the stationary condition on $\rho_\infty$ (fixed points of the dynamics) \eqref{eq:rkhs-gd-condition} translates to 
    \begin{align}
        \label{eq:station-gd}
         \mathcal{T} f_* (\Theta) = \left( \mathcal{T}\mathcal{T}^\star \frac{d\rho_{\infty}}{d|\rho_{\infty}|} \right) (\Theta), ~~\forall \Theta \in \text{supp}(\rho_\infty).
    \end{align}
    Here the function $\frac{d\rho_{\infty}}{d|\rho_\infty|} $ is the Radon-Nikodym derivative. In addition, one can easily verify that, as $\rho_{\infty}$ has bounded total variation
    \begin{align*}
        \frac{d\rho_{\infty}}{d|\rho_\infty|} \in L^2_{|\rho_{\infty}|}. 
    \end{align*}
    
    Therefore, combining all the above, one knows that
    \begin{align*}
        f_{\infty}(x) = \int \|\Theta\|\sigma(x^T\Theta)\rho_{\infty}(d\Theta)= \mathcal{T}^\star\frac{d\rho_{\infty}}{d|\rho_\infty|} \in \mathcal{H}_\infty
    \end{align*}
    and that for any $p\in L^2_{|\rho_\infty|}$
    \begin{align}
        \langle f_* - f_{\infty}, \mathcal{T}^* p \rangle_{\mu}  &= \langle \mathcal{T} (f_* - f_{\infty}),  p \rangle_{|\rho_\infty|} \\
        & = \left\langle \mathcal{T} f_* - \mathcal{T}\mathcal{T}^\star \frac{d\rho_{\infty}}{d|\rho_{\infty}|}, p \right\rangle_{|\rho_\infty|} \\
        & = \int \left(\mathcal{T} f_* - \mathcal{T}\mathcal{T}^\star \frac{d\rho_{\infty}}{d|\rho_{\infty}|} \right) (\Theta) |\rho_{\infty}|(d\Theta) =0 \quad \quad \text{due to \eqref{eq:station-gd}}
    \end{align}
 We have proved that $f_\infty = \mathcal{T}^\star\frac{d\rho_{\infty}}{d|\rho_\infty|}$ satisfies normal condition for being a minimizer to \eqref{eq:rkhs-proj}.
    \end{proof}

\begin{proof}[Proof of Proposition \ref{prop:kernel-compare}]

    The first inequality in \eqref{ieq:pos-compare} is trivial. For the second inequality, it suffices to show for any $c = (c_1,\dots,c_p)^T$, $x_1,\dots,x_p$, $\Theta$, we have
    \begin{align}
         \sum_{i,j} c_ic_j\|\Theta\|^2x_i^Tx_j\mathbbm{1}_{x_i^T \Theta > 0} \mathbbm{1}_{x_j^T \Theta > 0}  \geq \sum_{i,j}c_ic_j\sigma(x_i^T\Theta)\sigma(x_j^T\Theta)
    \end{align}
    The RHS equals
    \begin{align}
    &\quad\,\,\sum_{i,j}c_ic_jx_i^T\Theta x_j^T\Theta \mathbbm{1}_{x_i^T \Theta > 0} \mathbbm{1}_{x_j^T \Theta > 0} = \left( \sum_i c_i x_i^T\Theta \mathbbm{1}_{x_i^T \Theta > 0}\right)^2 \\
    &= \langle \Theta, \sum_i c_ix_i \mathbbm{1}_{x_i^T \Theta > 0}\rangle^2 \leq \|\Theta\|^2 \left\|\sum_ic_i x_i\mathbbm{1}_{x_i^T \Theta > 0}\right\|^2 = \text{LHS}.
    \end{align}
    For the last inequality, with compactness condition on $\rho_\infty$, we have
    \begin{align}
    \sum_{i,j}c_ic_j\int \|\Theta\|^2 \sigma(x_i^T\Theta)\sigma(x_j^T\Theta) |\rho_{\infty}|(\Theta)\leq D^2\sum_{i,j}c_ic_j \int \sigma(x_i^T\Theta)\sigma(x_j^T\Theta)|\rho_{\infty}|(\Theta).
    \end{align}
    Therefore, $D^2K_\infty^{(1)} \succeq H_\infty$.
    
\end{proof}

\begin{proof}[Proof of Theorem~\ref{thm:gap-decomposition}]
    
    Let us rewrite Corollary~\ref{cor:signed-kernel} into
    \begin{align}
        \label{eq:GD-rkhs}
         \frac{d}{dt} \| \Delta_t \|_\mu^2 = - 2 \langle \Delta_t, \mathcal{K}_t  \Delta_t \rangle_{\mu} = - 2 \| \mathcal{K}_t^{1/2} \Delta_t \|_{\mu}^2,
    \end{align}
    here $\mathcal{K}_t: L^2_\mu(x) \rightarrow L^2_\mu(x)$ denotes the integral operator associated with $K_t$,
    \begin{align}
        (\mathcal{K}_t f)(x) := \int  K_t( x, \tilde{x} ) f(\tilde{x})\mu(d\tilde{x}).
    \end{align}

    From \eqref{eq:GD-rkhs}
    \begin{align}
        \frac{d}{dt} \| \Delta_\infty \|_\mu^2 =  - 2 \| \mathcal{K}_\infty^{1/2} \Delta_\infty \|_{\mu}^2, 
    \end{align}
    we know that the RHS equals zero implies
    \begin{align*}
        \| \mathcal{K}_\infty^{1/2} \Delta_\infty \|_{\mu}^2 &= 0 \\
        \langle \mathcal{K}_\infty^{1/2} g,  \Delta_\infty \rangle_\mu = \langle g, \mathcal{K}_\infty^{1/2} \Delta_\infty \rangle_\mu &= 0, \quad \forall g \in L^2_\mu.
    \end{align*}
    This further implies $\Delta_\infty$ lies in the kernel of RKHS $\mathcal{K}_\infty$ as $\mathcal{K}_\infty = \{ \mathcal{K}_\infty^{1/2} g: g \in L^2_\mu\}$.
\end{proof}

\begin{proof}[Proof of Proposition~\ref{prop:interpolation}]
    The gradients on the original parameters are,
    \begin{align*}
        \frac{dw_j(t)}{dt} = -\widehat{\mathbf{E}}\left[\frac{\partial \ell(\by,f_t)}{\partial f}\sigma(\bx^Tu_j(t))\right] - \frac{1}{m}\lambda w_j(t),\\
        \frac{du_j(t)}{dt} = -\widehat{\mathbf{E}}\left[\frac{\partial \ell(\by,f_t)}{\partial f}w_j(t)\mathbbm{1}_{\bx^Tu_j(t)\geq0}\bx\right] - \frac{1}{m}\lambda u_j(t).
    \end{align*}
    Clearly, on the rescaled parameter, the following holds
 \begin{align*}
    \frac{d\theta_{j}}{dt} &= \sqrt{m} \frac{d w_{j}}{dt}= -\widehat{\mathbf{E}}\left[(f_t(\bx) - \by) \sigma(\bx^T\Theta_j(t))\right] - \frac{1}{m} \lambda \theta_{j}, \\
    \frac{d\Theta_{j}}{dt} &= \sqrt{m} \frac{d u_{j}}{dt}= -\widehat{\mathbf{E}}\left[(f_t(\bx) - \by) \theta_{j}\mathbbm{1}_{\bx^T\Theta_{j}\geq0}\bx\right] - \frac{1}{m} \lambda \Theta_j.
\end{align*} 
Multiply the first equation by $\theta_j$, and the second equation by $\theta_j^T$, take the difference, we can verify that    
\begin{align}
    \frac{d ( \theta_j^2 - \| \Theta_j \|^2)}{dt}=  - \lambda/m ( \theta_j^2 - \| \Theta_j \|^2) \\
    \theta_j(t)^2 - \| \Theta_j(t) \|^2 = \left( \theta_j(0)^2 - \| \Theta_j(0) \|^2 \right) \exp(-\lambda t/m ) \enspace. \label{eq:balanced}
\end{align}
Therefore the balanced condition still holds at stationarity for arbitrary bounded initialization,
\begin{align*}
    \theta_j(\infty)^2 - \| \Theta_j(\infty) \|^2 = 0, \forall j.
\end{align*}

Now the optimality condition for the velocity field reads the following, for any $\Theta_j(\infty) \in {\rm supp}(\widehat{\rho}_\infty^{\lambda})$ (we abbreviate the $\infty$ in the following display, note $\tilde{\theta}(\infty)$ corresponds to the second layer weights w.r.t. to $\Tilde{\Theta}(\infty)$)
\begin{align*}
    \theta_j \widehat{\E}[\by \mathbbm{1}_{\bx^T\Theta_{j}\geq0}\bx] = \theta_j \int \tilde{\theta} \widehat{\E}[\mathbbm{1}_{\bx^T\Theta_{j}\geq0}\bx \sigma(\bx^T \tilde{\Theta}) ] |\widehat{\rho}_\infty^{\lambda}|(d \tilde{\Theta}) + \frac{1}{m}\lambda \Theta_j \\
    \text{Multiply by $\Theta_j^T$,} ~~ \theta_j \widehat{\E}[\by \sigma(\bx^T \Theta_{j})] = \int \theta_j \tilde{\theta} \widehat{\E}[ \sigma(\bx^T \Theta_{j}) \sigma(\bx^T \tilde{\Theta}) ] |\widehat{\rho}_\infty^{\lambda}|(d \tilde{\Theta}) + \frac{\lambda}{m} \| \Theta_j \|^2 \\
    \theta_j  \widehat{\E}[\by \sigma(\bx^T \Theta_{j})] = \int \theta_j \tilde{\theta} \widehat{\E}[ \sigma(\bx^T \Theta_{j}) \sigma(\bx^T \tilde{\Theta}) ] |\widehat{\rho}_\infty^{\lambda}|(d \tilde{\Theta}) + \lambda  \int \theta_j \tilde{\theta} \mathbbm{1}_{\tilde{\Theta} = \Theta_j}  |\widehat{\rho}_\infty^{\lambda}|(d \tilde{\Theta})
\end{align*}
where the last step uses the condition $\theta_j^2(\infty) = \| \Theta_j(\infty) \|^2$, and the fact that $|\widehat{\rho}_\infty^{\lambda}| = \frac{1}{m} \sum_{j=1}^m \delta_{\Theta_j}$ and
\begin{align*}
    \int \theta_j \tilde{\theta} \mathbbm{1}_{\tilde{\Theta} = \Theta_j}  |\widehat{\rho}_\infty^{\lambda}|(d \tilde{\Theta}) = \frac{1}{m} \theta_j^2 = \frac{1}{m} \| \Theta_j \|^2.
\end{align*}

In the matrix form, where $\widehat{\rho}^{\lambda}_\infty = \frac{1}{m}\sum_{l\in [m]} {\rm sgn}(\theta_l) \delta_{\Theta_l} $
\begin{align*}
    \sum_{l \in [m]} \left[ n \widehat{U}(\Theta_j, \Theta_l) + n \lambda \mathbb{I}_{\Theta_l = \Theta_j} \right] \theta_l/m = \sigma(\Theta_j^T X) Y .
\end{align*}
Therefore, define $\sigma(x^T \Xi) := [\sigma(x^T\Theta_1) \ldots, \sigma(x^T\Theta_m)] \in \mathbb{R}^{1 \times m}$, and $\sigma(X \Xi) := [\sigma(x_1^T \Xi)^T,\ldots, \sigma(x_n^T \Xi)^T] \in \mathbb{R}^{m \times n}$, we have
\begin{align*}
    \widehat{f}^{{\rm nn}, \lambda}_{\infty} (x) &= \sum_{l \in [m]}  \theta_l \sigma(x^T \Theta_l) / m  = \sigma(x^T \Xi) [ \sigma(X \Xi)\sigma(X \Xi)^T + n\lambda I_m]^{-1} \sigma(X \Xi) Y \\
    & = \sigma(x^T \Xi) \sigma(X \Xi) [ \sigma(X \Xi)^T \sigma(X \Xi) + n\lambda I_n]^{-1} Y \\
    &= \widehat{H}_{\infty}^{\lambda}(x, X) \left[  \widehat{H}_{\infty}^{\lambda} (X, X) + n/m \cdot \lambda I_n\right]^{-1} Y \enspace.
\end{align*}
The last line follows as $\widehat{H}^{\lambda}(x, \tilde{x}):= \int \sigma(x^T \Theta) \sigma(\tilde{x}^T \Theta) |\widehat{\rho}^{\lambda}_\infty|(d\Theta) = 1/m \cdot \sigma(x^T \Xi) \sigma(\tilde{x}^T \Xi)^T$.

\end{proof}

\begin{proof}[Proof of Proposition~\ref{prop:adaptive-generalization}]
    \begin{align*}
        \| \lim_{\lambda \rightarrow 0} \widehat{f}^{{\rm nn}, \lambda}_{\infty}- f_*\|^2_{\mu} & \precsim \| \widehat{f}^{\rm rkhs}_{\infty} - f^{\rm rkhs}_{\infty}  \|_\mu^2 + \| f^{\rm rkhs}_{\infty}  - f_* \|_{\mu}^2 \\
        \| f^{\rm rkhs}_{\infty}  - f_* \|_{\mu}^2  & = \| H_{\infty}(x, X) H_{\infty}(X, X)^{+} [Y - f_*(X) + f_*(X) - f_\infty(X) + f_\infty(X)] - f_*(x)\|_\mu^2 \\
        & \precsim \| H_{\infty}(x, X) H_{\infty}(X, X)^{+} (Y - f_*(X)) \|_{\mu}^2 \\
        & \quad \quad + \E_{\bx \sim \mu} \langle H_{\infty}(X, X)^{+}H_{\infty}(X, \bx) , f_*(X) - f_\infty(X)\rangle^2 \\
        & \quad \quad + \| H_{\infty}(x, X) H_{\infty}(X, X)^{+} f_\infty(X) - f_\infty(x) \|_{\mu}^2 + \|f_\infty(x)  - f_*(x) \|_{\mu}^2 \enspace.
    \end{align*}
    For the first term, we can upper bound by $\sigma^2 \E_{\bx \sim \mu} \| H_\infty(X, X)^{-1} H_\infty(X, \bx)\|^2$. The second term can be upper bounded by 
    \begin{align*}
        \E_{\bx \sim \mu} \| H_\infty(X, X)^{-1} H_\infty(X, \bx)\|^2 \cdot n \|f_\infty(x)  - f_*(x) \|_{\hat{\mu}}^2.
    \end{align*}
    Proof is completed.
\end{proof}

\section*{Acknowledgement}
    We thank Maxim Raginsky for pointing out relevant references, and for providing helpful discussion. 

\bibliographystyle{Chicago}
\bibliography{ref}

\begin{thebibliography}{39}
\providecommand{\natexlab}[1]{#1}
\providecommand{\url}[1]{\texttt{#1}}
\expandafter\ifx\csname urlstyle\endcsname\relax
  \providecommand{\doi}[1]{doi: #1}\else
  \providecommand{\doi}{doi: \begingroup \urlstyle{rm}\Url}\fi

\bibitem[Anthony and Bartlett(2009)]{anthony2009neural}
Martin Anthony and Peter~L Bartlett.
\newblock \emph{Neural network learning: Theoretical foundations}.
\newblock cambridge university press, 2009.

\bibitem[Bach(2017)]{bach2017breaking}
Francis Bach.
\newblock Breaking the curse of dimensionality with convex neural networks.
\newblock \emph{Journal of Machine Learning Research}, 18\penalty0
  (19):\penalty0 1--53, 2017.

\bibitem[Barron et~al.(2008)Barron, Cohen, Dahmen, DeVore,
  et~al.]{barron2008approximation}
Andrew~R Barron, Albert Cohen, Wolfgang Dahmen, Ronald~A DeVore, et~al.
\newblock Approximation and learning by greedy algorithms.
\newblock \emph{The annals of statistics}, 36\penalty0 (1):\penalty0 64--94,
  2008.

\bibitem[Belkin et~al.(2018{\natexlab{a}})Belkin, Hsu, Ma, and
  Mandal]{belkin2018reconciling}
Mikhail Belkin, Daniel Hsu, Siyuan Ma, and Soumik Mandal.
\newblock Reconciling modern machine learning and the bias-variance trade-off.
\newblock \emph{arXiv preprint arXiv:1812.11118}, 2018{\natexlab{a}}.

\bibitem[Belkin et~al.(2018{\natexlab{b}})Belkin, Ma, and
  Mandal]{belkin2018understand}
Mikhail Belkin, Siyuan Ma, and Soumik Mandal.
\newblock To understand deep learning we need to understand kernel learning.
\newblock \emph{arXiv preprint arXiv:1802.01396}, 2018{\natexlab{b}}.

\bibitem[Bengio et~al.(2006)Bengio, Roux, Vincent, Delalleau, and
  Marcotte]{bengio2006convex}
Yoshua Bengio, Nicolas~L Roux, Pascal Vincent, Olivier Delalleau, and Patrice
  Marcotte.
\newblock Convex neural networks.
\newblock In \emph{Advances in neural information processing systems}, pages
  123--130, 2006.

\bibitem[Casselman(2014)]{casselman2014essays}
Bill Casselman.
\newblock Essays in analysis.
\newblock 2014.

\bibitem[Chizat and Bach(2018)]{chizat2018note}
Lenaic Chizat and Francis Bach.
\newblock A note on lazy training in supervised differentiable programming.
\newblock \emph{arXiv preprint arXiv:1812.07956}, 2018.

\bibitem[Cho and Saul(2009)]{cho2009kernel}
Youngmin Cho and Lawrence~K Saul.
\newblock Kernel methods for deep learning.
\newblock In \emph{Advances in neural information processing systems}, pages
  342--350, 2009.

\bibitem[Cybenko(1989)]{cybenko1989approximation}
George Cybenko.
\newblock Approximation by superpositions of a sigmoidal function.
\newblock \emph{Mathematics of control, signals and systems}, 2\penalty0
  (4):\penalty0 303--314, 1989.

\bibitem[Daniely et~al.(2016)Daniely, Frostig, and Singer]{daniely2016toward}
Amit Daniely, Roy Frostig, and Yoram Singer.
\newblock Toward deeper understanding of neural networks: The power of
  initialization and a dual view on expressivity.
\newblock In \emph{Advances In Neural Information Processing Systems}, pages
  2253--2261, 2016.

\bibitem[Du et~al.(2018)Du, Zhai, Poczos, and Singh]{du2018gradient}
Simon~S Du, Xiyu Zhai, Barnabas Poczos, and Aarti Singh.
\newblock Gradient descent provably optimizes over-parameterized neural
  networks.
\newblock \emph{arXiv preprint arXiv:1810.02054}, 2018.

\bibitem[Farrell et~al.(2018)Farrell, Liang, and Misra]{farrell2018deep}
Max~H Farrell, Tengyuan Liang, and Sanjog Misra.
\newblock Deep neural networks for estimation and inference: Application to
  causal effects and other semiparametric estimands.
\newblock \emph{arXiv preprint arXiv:1809.09953}, 2018.

\bibitem[Geman and Hwang(1982)]{geman1982nonparametric}
Stuart Geman and Chii-Ruey Hwang.
\newblock Nonparametric maximum likelihood estimation by the method of sieves.
\newblock \emph{The Annals of Statistics}, pages 401--414, 1982.

\bibitem[Ghorbani et~al.(2019)Ghorbani, Mei, Misiakiewicz, and
  Montanari]{ghorbani2019linearized}
Behrooz Ghorbani, Song Mei, Theodor Misiakiewicz, and Andrea Montanari.
\newblock Linearized two-layers neural networks in high dimension.
\newblock \emph{arXiv preprint arXiv:1904.12191}, 2019.

\bibitem[Hastie et~al.(2019)Hastie, Montanari, Rosset, and
  Tibshirani]{hastie2019surprises}
Trevor Hastie, Andrea Montanari, Saharon Rosset, and Ryan~J Tibshirani.
\newblock Surprises in high-dimensional ridgeless least squares interpolation.
\newblock \emph{arXiv preprint arXiv:1903.08560}, 2019.

\bibitem[Hornik et~al.(1989)Hornik, Stinchcombe, and
  White]{hornik1989multilayer}
Kurt Hornik, Maxwell Stinchcombe, and Halbert White.
\newblock Multilayer feedforward networks are universal approximators.
\newblock \emph{Neural networks}, 2\penalty0 (5):\penalty0 359--366, 1989.

\bibitem[Huang et~al.(2008)Huang, Cheang, and Barron]{huang2008risk}
Cong Huang, Gerald~HL Cheang, and Andrew~R Barron.
\newblock \emph{Risk of penalized least squares, greedy selection and
  l1-penalization for flexible function libraries}.
\newblock PhD thesis, Yale University, 2008.

\bibitem[Jacot et~al.(2018)Jacot, Gabriel, and Hongler]{jacot2018neural}
Arthur Jacot, Franck Gabriel, and Cl{\'e}ment Hongler.
\newblock Neural tangent kernel: Convergence and generalization in neural
  networks.
\newblock In \emph{Advances in neural information processing systems}, pages
  8571--8580, 2018.

\bibitem[Jones(1992)]{jones1992simple}
Lee~K Jones.
\newblock A simple lemma on greedy approximation in hilbert space and
  convergence rates for projection pursuit regression and neural network
  training.
\newblock \emph{The annals of Statistics}, 20\penalty0 (1):\penalty0 608--613,
  1992.

\bibitem[Jordan et~al.(1998)Jordan, Kinderlehrer, and
  Otto]{jordan1998variational}
Richard Jordan, David Kinderlehrer, and Felix Otto.
\newblock The variational formulation of the fokker--planck equation.
\newblock \emph{SIAM journal on mathematical analysis}, 29\penalty0
  (1):\penalty0 1--17, 1998.

\bibitem[Koehler and Risteski(2018)]{koehler2018representational}
Frederic Koehler and Andrej Risteski.
\newblock Representational power of relu networks and polynomial kernels:
  Beyond worst-case analysis.
\newblock \emph{arXiv preprint arXiv:1805.11405}, 2018.

\bibitem[Liang and Rakhlin(2018)]{liang2018just}
Tengyuan Liang and Alexander Rakhlin.
\newblock Just interpolate: Kernel" ridgeless" regression can generalize.
\newblock \emph{The Annals of Statistics, to appear}, 2018.

\bibitem[Ma et~al.(2017)Ma, Bassily, and Belkin]{ma2017power}
Siyuan Ma, Raef Bassily, and Mikhail Belkin.
\newblock The power of interpolation: Understanding the effectiveness of sgd in
  modern over-parametrized learning.
\newblock \emph{arXiv preprint arXiv:1712.06559}, 2017.

\bibitem[Maennel et~al.(2018)Maennel, Bousquet, and Gelly]{maennel2018gradient}
Hartmut Maennel, Olivier Bousquet, and Sylvain Gelly.
\newblock Gradient descent quantizes relu network features.
\newblock \emph{arXiv preprint arXiv:1803.08367}, 2018.

\bibitem[Mei et~al.(2018)Mei, Montanari, and Nguyen]{mei2018mean}
Song Mei, Andrea Montanari, and Phan-Minh Nguyen.
\newblock A mean field view of the landscape of two-layers neural networks.
\newblock \emph{arXiv preprint arXiv:1804.06561}, 2018.

\bibitem[Niyogi and Girosi(1996)]{niyogi1996relationship}
Partha Niyogi and Federico Girosi.
\newblock On the relationship between generalization error, hypothesis
  complexity, and sample complexity for radial basis functions.
\newblock \emph{Neural Computation}, 8\penalty0 (4):\penalty0 819--842, 1996.

\bibitem[Park and Sandberg(1991)]{park1991universal}
Jooyoung Park and Irwin~W Sandberg.
\newblock Universal approximation using radial-basis-function networks.
\newblock \emph{Neural computation}, 3\penalty0 (2):\penalty0 246--257, 1991.

\bibitem[Poggio et~al.(2017)Poggio, Mhaskar, Rosasco, Miranda, and
  Liao]{poggio2017and}
Tomaso Poggio, Hrushikesh Mhaskar, Lorenzo Rosasco, Brando Miranda, and Qianli
  Liao.
\newblock Why and when can deep-but not shallow-networks avoid the curse of
  dimensionality: a review.
\newblock \emph{International Journal of Automation and Computing}, 14\penalty0
  (5):\penalty0 503--519, 2017.

\bibitem[Rahimi and Recht(2008)]{rahimi2008random}
Ali Rahimi and Benjamin Recht.
\newblock Random features for large-scale kernel machines.
\newblock In \emph{Advances in neural information processing systems}, pages
  1177--1184, 2008.

\bibitem[Rahimi and Recht(2009)]{rahimi2009weighted}
Ali Rahimi and Benjamin Recht.
\newblock Weighted sums of random kitchen sinks: Replacing minimization with
  randomization in learning.
\newblock In \emph{Advances in neural information processing systems}, pages
  1313--1320, 2009.

\bibitem[Rakhlin and Zhai(2018)]{rakhlin2018consistency}
Alexander Rakhlin and Xiyu Zhai.
\newblock Consistency of interpolation with laplace kernels is a
  high-dimensional phenomenon.
\newblock \emph{arXiv preprint arXiv:1812.11167}, 2018.

\bibitem[Rotskoff and Vanden-Eijnden(2018)]{rotskoff2018neural}
Grant~M Rotskoff and Eric Vanden-Eijnden.
\newblock Neural networks as interacting particle systems: Asymptotic convexity
  of the loss landscape and universal scaling of the approximation error.
\newblock \emph{arXiv preprint arXiv:1805.00915}, 2018.

\bibitem[Rudi and Rosasco(2017)]{rudi2017generalization}
Alessandro Rudi and Lorenzo Rosasco.
\newblock Generalization properties of learning with random features.
\newblock In \emph{Advances in Neural Information Processing Systems}, pages
  3215--3225, 2017.

\bibitem[Sirignano and Spiliopoulos(2019)]{sirignano2019mean}
Justin Sirignano and Konstantinos Spiliopoulos.
\newblock Mean field analysis of neural networks: A central limit theorem.
\newblock \emph{Stochastic Processes and their Applications}, 2019.

\bibitem[Stone(1980)]{stone1980optimal}
Charles~J Stone.
\newblock Optimal rates of convergence for nonparametric estimators.
\newblock \emph{The annals of Statistics}, pages 1348--1360, 1980.

\bibitem[Vapnik(1998)]{vapnik1998statistical}
Vladimir Vapnik.
\newblock \emph{Statistical learning theory. 1998}, volume~3.
\newblock Wiley, New York, 1998.

\bibitem[Wahba(1990)]{wahba1990spline}
Grace Wahba.
\newblock \emph{Spline models for observational data}, volume~59.
\newblock Siam, 1990.

\bibitem[Zhang et~al.(2016)Zhang, Bengio, Hardt, Recht, and
  Vinyals]{zhang2016understanding}
Chiyuan Zhang, Samy Bengio, Moritz Hardt, Benjamin Recht, and Oriol Vinyals.
\newblock Understanding deep learning requires rethinking generalization.
\newblock \emph{arXiv preprint arXiv:1611.03530}, 2016.

\end{thebibliography}

\newpage
\appendix

\clearpage
\setcounter{page}{1}

\section{Appendix}

\subsection{Supporting Results}

\begin{proof}[Proof of Lemma~\ref{lem:pdf-sign-measure}]
	Let's first show that in the infinite neuron limit $m \rightarrow \infty$, $\rho_{+, t}, \rho_{-, t}$ are properly defined. Therefore Eqn.~\eqref{def:training-infinite} in the above theorem also characterize the distribution dynamics for infinite neurons NN, induced by gradient flow training. For simplicity, we assume the initialization $\rho_{+,0},\rho_{-,0}$ with bounded support. We add the superscript $m$, $\rho_{+,t}^m, \rho_{-,t}^m,\rho_{t}^m$ to \eqref{eq:sign-measure} to indicate their dependence on $m$. Consider $\nabla_\Theta V$, $\nabla_{\Theta}U(\Theta,\tilde \Theta)$ in \eqref{eq:velocity} are bounded and uniform Lipchitz continuous as in \cite[A3]{mei2018mean}. With the same proof as in \cite[Theorem 3]{mei2018mean}, one can show that with $m\rightarrow \infty$, the initial distribution $\rho^m_0 \xrightarrow{d} \tilde \rho_0 = \rho_{+,0} - \rho_{-,0}$ by law of large number, and by the solution's continuity depending on the initial value. Therefore we have $\rho_t^m \xrightarrow{d} \rho_t$ as $m\rightarrow \infty$ well defined.

The velocity of a particle $\Theta$ in the positive part as a rewrite of \eqref{def:training-rescale_1}-\eqref{def:training-rescale_4} is
\begin{align}
	\mathcal{V}(\Theta,\rho_t) = \|\Theta\| \left( \nabla _{\Theta}V(\Theta)+ \nabla_{\Theta}\int U(\Theta,\tilde\Theta)\|\tilde\Theta\|\rho_t(d \tilde\Theta)\right),
\end{align}
resp. for the negative part and \eqref{def:training-rescale_4}, we have
\[ -\mathcal{V}(\Theta,\rho_t) =  -\|\Theta\| \left( \nabla_{\Theta} V(\Theta)+ \nabla_{\Theta}\int U(\Theta,\tilde\Theta)\|\tilde\Theta\|\rho_t(d \tilde\Theta)\right).\]
Given the velocity of particle, we have the transport equation for gradient flow,
\begin{align*}
	\partial_t\rho_{+,t} = - \nabla_{\Theta}\cdot\left(\rho_{+,t}\cdot\mathcal{V}(\Theta,\rho_t)\right),\\
	\partial_t\rho_{-,t} = - \nabla_{\Theta}\cdot\left(-\rho_{-,t}\cdot\mathcal{V}(\Theta,\rho_t)\right).
\end{align*}
To see this, recall the definition of weak derivative $\partial_t \rho_t$: for any bounded smooth function $g$, $\partial_t \rho_t$ is defined in the following sense
\begin{align}
	d\cdot\int g \rho_t = -\int g \partial_t \rho_t \cdot dt.
\end{align}
We take any bounded smooth function $g(\Theta)$, given the velocity of $\Theta$'s , then we have 
\begin{align}
 -\int g \partial_t \rho_t \cdot dt = d\cdot \int g(\Theta)\rho_{+,t}(\Theta) = \int \nabla g(\Theta) \cdot \mathcal{V} (\Theta,\rho_t) \rho_{+,t}(\Theta) \cdot dt,
\end{align}
and $\rho_{-,t}$ correspondingly. By the weak derivative, we get the above PDE.
We use the above dynamic description as the training process for infinite neuron NN. Plug above equation into $\rho_t = \rho_{+,t}-\rho_{-,t}$ and $|\rho_t| = \rho_{+,t}+\rho_{-,t}$, we get
\begin{align} 
	\partial_t\rho_t(\Theta) = - \nabla_{\Theta} \cdot\left(|\rho_t|(\Theta) \mathcal{V}(\Theta, \rho_t)\right), \nonumber\\
	\partial_t|\rho_t|(\Theta) = - \nabla_\Theta \cdot\left(\rho_t(\Theta) \mathcal{V}(\Theta, \rho_t) \right).\label{eq:pde_sign_char}
\end{align}
\end{proof}

\begin{proof}[Proof of Proposition~\ref{lem:balance}]
	It suffices to show $\theta^2_{+,i}(t) = \|\Theta_{+,i}(t)\|_2^2$ and resp. $\theta^2_{-i}(t)=\|\Theta_{-,i}(t)\|_2^2$. By our path dynamics, we have
	\begin{align}
		\frac{d\theta^2_{+,i}}{dt} &= 2\theta_{+,i}\frac{d\theta_{+,i}}{dt} = -2\mathbf{E}_{\bd{z}}\left[\frac{\partial \ell(\by,f(\bx))}{\partial f}\theta_{+,i}\sigma(\bx^T\Theta_{+,i})\right],\\
		\frac{d\|\Theta_{+,i}\|_2^2}{dt}&=2\Theta_{+,i}^T\frac{d\Theta_{+,i}}{dt} = -2\mathbf{E}_{\bd{z}} \left[\frac{\partial \ell(\by,f(\bx))}{\partial f}\theta_{+,i}\mathbbm{1}_{\bx^T\Theta_{+,i}\geq0}\bx^T\Theta_{+,i}\right] = \frac{d\theta_{+,i}^2}{dt}.
	\end{align}
	Thus, by the initialization, we have $\theta_{+,i}(t) = \|\Theta_{+,i}(t)\|$,  and resp. $\theta_{-,i}(t) = -\|\Theta_{-,i}(t)\|$.
\end{proof}

\begin{proposition}[No sign change]
		    \label{lem:no-change-sign}
		    For the training process \eqref{eq:training_2} for problem \eqref{appr_prob} with NN \eqref{nn}, once $w_j(t)$ and $u_j(t)$ hit zero at $t_0$, for $t>t_0$ at least there exists a solution that can be viewed as training without the $j$-th neuron.
\end{proposition}

\begin{proof}[Proof of Proposition~\ref{lem:no-change-sign}]
  Using $w_j(t_0)$, $u_j(t_0)$, for $j\neq i$, as an initial value for ODE \eqref{eq:training_2} without the $i$-th node. By assumption, we have a solution of this $2\cdot(2m-1)$-dimensional initial value problem. Then padding the solution with $u_i \equiv 0$ and $w_i\equiv 0$, which can be a solution for ODE  \eqref{eq:training_2} with $i$-th neuron included.
\end{proof}

\begin{proof}[Proof of Lemma~\ref{lem:dynamic-kernel}]
	 First we write down the dynamic of prediction $f(\tilde{x})$ at each point $\tilde{x}$ based on Eqn. \eqref{eq:training_2}. For notational simplicity, let $u_j,w_j$ be $u_j(t),w_j(t)$, and let $o^{1}_j(\tilde{x}) = \sigma(u_j^T\tilde{x})$, and with the square loss $\ell(y, f) = \frac{1}{2}(y - f)^2$, we have
 		\begin{align*}		
 			\frac{df_t(\tilde{x})}{dt} &= \sum_{j=1}^m \left[ \frac{dw_j}{dt}o^{1}_j(\tilde{x})  +w_j\frac{do ^{1}_j(\tilde{x})}{dt} \right]\\	
			&= \sum_{j=1}^m \left\{  \mathbf{E}_{\bd{z}}\left[(\by - f_t(\bx))\sigma(\bx^Tu_j)\right] o^{1}_j(\tilde{x}) + w_j \mathbbm{1}_{\tilde{x}^Tu_j\geq0} \tilde{x}^T \mathbf{E}_{\bd{z}}\left[(\by - f_t(\bx))w_j \mathbbm{1}_{\bx^Tu_j\geq0}\bx \right]  \right\} \\
 			&= \sum_{j=1}^m \left\{ \mathbf{E}_{\bd{x}}\left[ \left(f_*(\bd{x})-f_t(\bd{x})\right) \left( \sigma(\tilde{x}^Tu_j)\sigma(\bx^Tu_j) + w^2_j\mathbbm{1}_{\tilde{x}^Tu_j\geq0}\mathbbm{1}_{\bx^Tu_j\geq0}\tilde{x}^T\bx  \right)   \right] \right\} \\		
 			&= \mathbf{E}_{\bd{x}}\Bigg\{ \sum_{j=1}^m \bigg[\sigma(\tilde{x}^Tu_j)\sigma(\bx^Tu_j)  + w^2_j\mathbbm{1}_{\tilde{x}^Tu_j\geq0}\mathbbm{1}_{\bx^Tu_j\geq0}\tilde{x}^T\bx \bigg]   \left(f_*(\bx)-f_t(\bx)\right) \Bigg\} \\		
 			&= \mathbf{E}_{\bd{x}} \left[ K_t(\tilde{x},\bd{x})\Delta_t(\bd{x}) \right]. 	
 		\end{align*}		
 		Therefore, we have		
 		\begin{align}\label{eq:residual}		
 			\frac{d\mathbf{E}_{ \bd{x}}\left[\frac{1}{2}\Delta_t(\bd{x})^2\right]}{dt} &= - \mathbf{E}_{ \bd{x}} \left[ (f_*(\bx) - f_t(\bx)) \frac{d f_t(\bd{x})}{dt} \right] \\
			&= - \mathbf{E}_{ \bd{x}} \left[ \Delta_t(\bx) \mathbb{E}_{\tilde{\bx}} \left[ K_t(\bd{x},\bd{\tilde x}) \Delta(\bd{\tilde x}) \right] \right] \nonumber \\
			&=  -\mathbf{E}_{\bd{x}, \bd{\tilde x}}\left[ \Delta_t(\bd{x})K_t(\bd{x}, \bd{\tilde x})\Delta_t(\bd{\tilde x})\right] \nonumber.
 		\end{align}
\end{proof}

\begin{proof}[Proof of Corollary~\ref{cor:general-loss}]
	The first equality follows from the proof in Lemma~\ref{lem:dynamic-kernel}. Recall the property for strongly convex function
	\begin{align}
		\ell(y_i, f_t(x_i)) - \ell(y_i, f_*(x_i)) \leq \frac{1}{2\alpha} \left[ \frac{\partial \ell(y_i, f_t(x_i))}{\partial f} \right]^2 = \frac{1}{2\alpha} \Delta_t(x_i)^2.
	\end{align}
	Therefore $-\mathbf{E}_{\bd{x}, \bd{\tilde x}}\left[ \Delta_t(\bd{x})K_t(\bd{x}, \bd{\tilde x})\Delta_t(\bd{\tilde x})\right] \leq -\frac{\lambda_t}{n} \sum_{i=1}^n \Delta_t(x_i)^2 \leq - 2\alpha\lambda_t \cdot \widehat{\E} \left[\ell(\by, f_t(\bx)) - \ell(\by, f_*(\bx))\right].$
\end{proof}

\begin{proof}[Proof of Lemma~\ref{lem:K_0}]
	We know
	\begin{align}
		\E_{\bd{u} \sim \pi} \left[ \sigma(\bd{u}^T x)  \sigma(\bd{u}^T \tilde x ) \right] &= 
		\E_{\bd{u} \sim \pi} \tilde x^T \left[ \bd{u} \bd{u}^T \mathbbm{1}_{\bd{u}^T x >0} \mathbbm{1}_{\bd{u}^T \tilde x >0} \right]   x
	\end{align}
	Consider the coordinate system $e_1, e_2, \ldots e_d$ such that $e_1, e_2$ spans the space of $x, \tilde x$, with 	\begin{align}
		x = e_1, \tilde x = \cos\theta \cdot e_1 + \sin\theta \cdot e_2,
	\end{align}
	where $\theta = \arccos(x^T\tilde x)$. Note $\bd{u} = [v_1, v_2, \ldots v_d]$ is still an isotropic Gaussian under this coordinate system.
	The constraint reads
	\begin{align}
		&\mathbbm{1}_{\bd{u}^T x >0} \mathbbm{1}_{\bd{u}^T \tilde x >0}, \\
		\text{equivalent to} \quad & v_1 >0 , ~v_1 \cos\theta + v_2 \sin\theta >0,
	\end{align}
	and one can see that $v_2, \ldots v_d$ integrate out.
	
	Let's focus on the spherical coordinates of $v_1 = r\cos\phi, v_2=r\sin\phi$, then $r^2 \sim \chi^2(2)$ and $\phi \sim U[-\pi, \pi]$. W.l.o.g., we can consider the case when $\theta \in [0, \pi]$.
	\begin{align*}
		&\E_{\bd{u} \sim \pi} \left[ \bd{u} \bd{u}^T \mathbbm{1}_{\bd{u}^T x >0} \mathbbm{1}_{\bd{u}^T \tilde x >0} \right] x\\
		&= \E[r^2] \left(e_1 \cdot  \frac{1}{2\pi}\int_{-\pi}^{\pi} \cos^2\phi \mathbbm{1}_{\phi \in [\theta-\pi/2, \pi/2]} d\phi + e_2 \cdot \frac{1}{2\pi}\int_{-\pi}^{\pi} \cos\phi \sin\phi \mathbbm{1}_{\phi \in [\theta-\pi/2, \pi/2]} d\phi  \right) \\
		& \text{because the above are equivalent to $e_1 \E[v_1^2 \mathbbm{1}_{\bd{u}^T x >0} \mathbbm{1}_{\bd{u}^T \tilde x >0}] + e_2 \E[v_1 v_2 \mathbbm{1}_{\bd{u}^T x >0} \mathbbm{1}_{\bd{u}^T \tilde x >0}]$} \\
		&= 2 \cdot \frac{1}{2\pi} \left[ e_1 \cdot \frac{\pi-\theta}{2} + ( e_1 \cos\theta + e_2\sin\theta) \cdot \frac{\sin\theta}{2}  \right] \\
		& \text{just evaluate $\int_{\theta-\pi/2}^{\pi/2} \cos^2 \phi d\phi$, $\int_{\theta-\pi/2}^{\pi/2} \cos \phi \sin \phi d\phi$} \\
		&= \frac{\pi-\theta}{2\pi} x + \frac{\sin\theta}{2\pi} \tilde x.
	\end{align*}
	Therefore, we get
	\begin{align*}
		&\E_{\bd{u} \sim \pi} x^T \left[ \bd{u} \bd{u}^T \mathbbm{1}_{\bd{u}^T x >0} \mathbbm{1}_{\bd{u}^T \tilde x >0} \right] \tilde x\\
		&= \frac{\pi-\theta}{2\pi} \cos\theta + \frac{\sin\theta}{2\pi}
	\end{align*}
	Similarly, we have
	\begin{align}
		&\E_{\bd{u} \sim\pi} \tilde x^T\left[ \mathbbm{1}_{\bd{u}^Tx\geq0}\mathbbm{1}_{\bd{u}^T\tilde x\geq 0 }\right]x = \frac{\pi-\theta}{2\pi}\cos\theta. 
	\end{align}
	Summing them up, we get the result.
\end{proof}

\begin{proof}[Proof of Corollary~\ref{cor:signed-kernel}]
	Our proof essentially follows the same steps for \eqref{def:gd-kernel}. First, we write down the dynamic of $f_t(x)$,
	\begin{align}
		\frac{df_t(x)}{dt} &=\frac{\int \|\Theta\|\sigma(x^T\Theta) \rho_{t}(d\Theta)}{dt}.
	\end{align}
	Plug-in the training dynamic \eqref{eq:pde_sign_char}, we get
	\begin{align*}
		\frac{df_t(x)}{dt} &= -\int -\nabla_{\Theta}\left[	 \|\Theta\|\sigma(x^T\Theta)\right] \cdot \mathcal{V} (\Theta, \rho_t) | \rho_t |(d\Theta)\\
		&= \int \nabla_{\Theta}\left[ \|\Theta\|\sigma(x^T\Theta) \right]\cdot \| \Theta \| \left\{ \mathbf{E_{\bd{\tilde x}}}[f_*(\bd{\tilde x})\mathbbm{1}_{\bd{\tilde x}^T\Theta \geq 0}\bd{\tilde x}] - \mathbf{E}_{\bd{\tilde x}} \left[\int \left( \|\tilde\Theta\|\sigma(\tilde \bx^T\tilde \Theta)\mathbbm{1}_{\bd{\tilde x}^T\Theta \geq 0}\bd{\tilde x} \right) \rho_t(d\tilde \Theta) \right] \right\} |\rho_{t}|(d\Theta) \nonumber \\
		&=\mathbf{E}_{\bd{\tilde x}} \left\{ \int   \nabla_{\Theta}\left[ \|\Theta\|\sigma(x^T\Theta) \right]\cdot  \| \Theta \|  \left[ \Delta_t(\bd{\tilde x}) \mathbbm{1}_{\bd{\tilde x}^T\Theta\geq0}\bd{\tilde x}\right]|\rho_t|(d\Theta)\right\}\\
		&= \mathbf{E}_{\bd{\tilde x}} \left\{\Delta_t(\bd{\tilde x}) \cdot \int \|\Theta\|^2 \mathbbm{1}_{x^T\Theta\geq0}\mathbbm{1}_{\bd{\tilde x}^T\Theta \geq 0}x^T\bd{\tilde x} + \sigma(x^T\Theta)\sigma(\bd{\tilde x}^T\Theta) |\rho_t|(d \Theta)\right\}.
 	\end{align*}
	Therefore, we have
	\begin{align}
		\frac{d\mathbf{E}_{\bd{x}}\left[ \frac{1}{2}\Delta_t(\bx)^2 \right]}{dt} = -\mathbf{E}_{\bd{x},\bd{\tilde x}}[\Delta_t(\bd{x})K_t(\bd{x},\bd{\tilde x})\Delta_t(\bd{\tilde x})].
	\end{align}
\end{proof}

\subsection{Extensions}

\label{sec:extension-MLP}

In this section, we extend the definition of the dynamic kernel in Section~\ref{sec:kernel} to the multi-layer neural networks case. 
We construct a recursive expression for the kernel defined by the multi-layer perceptron (MLP). Let $\Theta_{i,j}^{l}$, $l=0,\cdots, h-1$ denote the coefficient from the $i$-th node on the $l$-th layer to the $j$-th node on the $(l+1)$-th layer. Let the input (before activation) of the $i$-th node on $l$-th layer be $v_i^l(x) = \sum_{j}\Theta^{l-1}_{j,i}o^{l-1}_j(x)$ and let the output at that node be $o_i^l=\sigma(v_i^l)$, for $ l \notin\{0,h\}$, and $o_i^l = x_i$, for $l=0$. The final output $g(x) = (v^h_1(x),v^h_2(x),\cdots, v^h_{L_h}(x))^T$. Let $L_0 = d$ and $L_i$ is the number of nodes at the $i$-th layer. Denote $K_t^{h}(x,\tilde x;\{\Theta^{l}\}_{l=0,\dots,h})$ the kernel of $h$ layers NN. The training dynamic is still the gradient flow, for all $\Theta$
\[
	\frac{d\Theta}{dt} =- \mathbb{E}_{\bd{z}}\left[\frac{\partial\ell(\by,g(\bx))}{\partial g}\frac{\partial g(\bx)}{\partial \Theta}\right].
\]
\begin{proposition}
	\label{prop:mlp}
For a $(h+1)$-layer NN function denoted by $g(x)$, for simplicity, let 
\begin{align}
	K_t^{h+1}(x, \tilde x) = K_t^{h+1}(x,\tilde x;\{\Theta^{l}\}_{l=0,\dots,h+1}), \\
	K_t^h(z,\tilde z) = K_t^{h}(z,\tilde z;\{\Theta^{l}\}_{l=1,\dots,h+1}).
\end{align}
With gradient flow training process, we have the following recursive representation of the corresponding kernel matrix
\[K_t^{h+1}(x, \tilde x) = K_t^{h}(o^1(x),o^1(\tilde x))  +  \sum_{i=1,j=1}^{L_0,L_1}\frac{\partial g(x)}{\partial \Theta^0_{i,j}}\frac{\partial g(\tilde  x)}{\partial \Theta^0_{i,j}}. \]
Here the kernel matrix is always positive semidefinite.
\end{proposition}

\begin{proof}[Proof of Proposition~\ref{prop:mlp}]
For notational simplicity, let $K_t^{h+1}(x,\tilde x) = K_t^{h+1}(x,\tilde x;\{\Theta^{l}\}_{l=0,\dots,h+1})$, and $$K_t^{h}(z,\tilde z) = K_t^{h}(z,\tilde z;\{\Theta^{l}\}_{l=1,\dots,h+1}).$$ For the proof, we calculate the dynamic of prediction $g(x)$, by elementary calculus, we have
	\begin{align}
		\frac{dg(x)}{dt} &= -\mathbf{E}_{\bx}[f_*(\bx) - g(\bx)] \Big[ \sum_{\text{all }\Theta}\frac{\partial g(x)}{\partial \Theta}\cdot\frac{\partial g(\bx)}{\partial\Theta}\Big].
	\end{align}
	With same calculation for the dynamic of $\Delta_t$ as in \eqref{eq:residual}, we get
	\begin{align}
		K_t^{h+1}(x,x') =  \sum_{\Theta \in \Theta^0}\frac{\partial g(x)}{\partial \Theta}\cdot\frac{\partial g(\tilde x)}{\partial \Theta} + \sum_{\text{other }\Theta}\frac{\partial g(x)}{\partial \Theta}\cdot\frac{\partial g(\tilde x)}{\partial \Theta}.
	\end{align}
	By induction, we get
	\begin{align}
	K_t^{h+1}(x,\tilde x) = K_t^{h}(o^1(x),o^1(\tilde x)) +  \sum_{i=1,j=1}^{L_0,L_1}\frac{\partial g(x)}{\partial \Theta^0_{i,j}}\frac{\partial g(\tilde x)}{\partial \Theta^0_{i,j}}.
	\end{align}
Now, we prove the positive semi-definiteness of the kernel. By induction, we only need to prove that the second term above is non-negative. We construct a canonical mapping $\phi_{h+1}(x): = v(x), \mathbb{R}^d \rightarrow \mathbb{R}^{L_0 \times L_1}$, whereas the $i,j$-th coordinate $v(x)_{i,j} = \frac{\partial g(x)}{\partial \Theta^0_{i,j}}$. Then the second term can be seen as a inner product $\langle \phi_{h+1}(x),\phi_{h+1}(\tilde x)\rangle$, which implies the non-negativity. 
\end{proof}
\end{document}